\documentclass[lettersize,journal]{IEEEtran}
\usepackage{amsmath,amsfonts}
\usepackage{algorithmic}
\usepackage{algorithm}
\usepackage{array}
\usepackage[caption=false,font=normalsize,labelfont=sf,textfont=sf]{subfig}
\usepackage{textcomp}
\usepackage{stfloats}
\usepackage{url}
\usepackage{verbatim}
\usepackage{graphicx}
\usepackage{cite}
\usepackage{xcolor-material, colortbl, soul, wrapfig}
\usepackage{nicefrac, xspace, multirow, makecell}
\usepackage{subfig, pifont, arydshln}
\usepackage{tcolorbox}
\usepackage{hyperref}
\usepackage{enumitem}
\usepackage{amsmath}
\usepackage{amssymb}
\usepackage{mathtools}
\usepackage{amsthm}
\usepackage{pifont}
\usepackage{booktabs}
\usepackage[capitalize,noabbrev]{cleveref}
\usepackage{eqparbox,empheq}
\usepackage[normalem]{ulem}

\theoremstyle{plain}
\newtheorem{theorem}{Theorem}[section]

\newtheorem{lemma}[theorem]{Lemma}
\newtheorem{corollary}[theorem]{Corollary}
\theoremstyle{definition}

\newtheorem{assumption}[theorem]{Assumption}
\theoremstyle{remark}

\newtheorem{finding}[theorem]{Finding}
\definecolor{Gray}{gray}{0.9}
\usepackage{algorithmic}

\newtcolorbox{mycolorbox}{
            colback=Cerulean!30,  %
            boxrule=0.5pt,      %
}

\makeatletter
\DeclareRobustCommand\onedot{\futurelet\@let@token\@onedot}
\def\@onedot{\ifx\@let@token.\else.\null\fi\xspace}

\def\eg{\emph{e.g}\onedot} 
\def\ie{\emph{i.e}\onedot}

\def\wrt{w.r.t\onedot} 

\makeatother

\makeatletter
 \def\SOUL@hlpreamble{%
 \setul{}{3.2ex}%
 \let\SOUL@stcolor\SOUL@hlcolor
 \SOUL@stpreamble
 }
\makeatother

\definecolor{difcolor}{RGB}{204, 0, 204}

\DeclareRobustCommand{\fedavg}
{\textsc{FedAvg}}

\DeclareRobustCommand{\fedavgm}
{\textsc{FedAvgM}}

\DeclareRobustCommand{\fedcm}
{\textsc{FedCM}}

\DeclareRobustCommand{\cifar}[1]
{\textsc{Cifar-#1}}

\DeclareRobustCommand{\resnet}
{\textsc{ResNet-20}}
\DeclareRobustCommand{\lenet}
{\textsc{CNN}}

\DeclareRobustCommand{\ghb}
{\textsc{GHBM}}

\DeclareRobustCommand{\S}[1]
{\mathcal{S}^{#1}}

\DeclareRobustCommand{\normsq}[1]
{\left\|#1\right\|^2}

\crefname{equation}{Eq.}{Eq.}
\Crefname{equation}{Equation}{Equations}
\creflabelformat{equation}{(#2#1#3)}

\crefname{figure}{Fig.}{Figs.}
\Crefname{figure}{Figure}{Figures}

\crefname{table}{Tab.}{Tabs.}
\Crefname{table}{Table}{Tables}

\crefname{chapter}{Chap.}{Chaps.}
\Crefname{chapter}{Chapter}{Chapters}

\crefname{section}{Sec.}{Secs.}
\Crefname{section}{Section}{Sections}

\crefname{subsection}{Sec.}{Secs.}
\Crefname{subsection}{Section}{Sections}

\crefname{subsubsection}{Sec.}{Secs.}
\Crefname{subsubsection}{Section}{Sections}

\crefname{paragraph}{Par.}{Pars.}
\Crefname{paragraph}{Paragraph}{Paragraphs}

\crefname{prob}{Prob.}{Probs.}
\Crefname{prob}{Problem}{Problems}

\crefname{property}{Prop.}{Props.}
\Crefname{property}{Property}{Properties}

\crefname{constr}{Constraint}{Constraints} %
\Crefname{constr}{Constraint}{Constraints}

\crefname{algocf}{Algorithm}{Algorithms}
\Crefname{algocf}{Algorithm}{Algorithms}

\crefname{algorithm}{Algorithm}{Algorithms}
\Crefname{algorithm}{Algorithm}{Algorithms}

\crefname{hyp}{Hyp.}{Hyps.}
\Crefname{hyp}{Hypothesis}{Hypothesis}

\crefname{assumption}{Assumption}{Assumptions}
\Crefname{assumption}{Assumption}{Assumptions}

\crefname{lem}{Lem.}{Lems.}
\Crefname{lem}{Lemma}{Lemma}

\crefname{prop}{Prop.}{Props.}
\Crefname{prop}{Proposition}{Propositions}

\crefname{theorem}{Thm.}{Thms.}
\Crefname{theorem}{Thm.}{Thms.}

\usepackage{aligned-overset}
\usepackage{parskip}

\begin{document}

\title{On the Limits of Momentum in \\ Decentralized and Federated Optimization}

\author{\IEEEauthorblockN{Riccardo Zaccone\IEEEauthorrefmark{1}\thanks{\IEEEauthorrefmark{1} Corresponding author},
Sai Praneeth Karimireddy \IEEEauthorrefmark{2} and Carlo Masone \IEEEauthorrefmark{3}}\\
\IEEEauthorblockA{\IEEEauthorrefmark{1}\IEEEauthorrefmark{2}Department of Computer and Control Engineering,
Polytechnic of Turin\\
\IEEEauthorrefmark{3}Thomas Lord Department of Computer Science, USC Viterbi School of Engineering\\
Email: \IEEEauthorrefmark{1}riccardo.zaccone@polito.it,
\IEEEauthorrefmark{2}karimire@usc.edu,
\IEEEauthorrefmark{3}carlo.masone@polito.it
}}

\maketitle

\begin{abstract}
Recent works have explored the use of momentum in local methods to enhance distributed SGD.
This is particularly appealing in Federated Learning (FL), where momentum intuitively appears as a solution to mitigate the effects of statistical heterogeneity.
Despite recent progress in this direction, it is still unclear if momentum can guarantee convergence under unbounded heterogeneity in decentralized scenarios, where only some workers participate at each round.
In this work we analyze momentum under cyclic client participation, and theoretically prove that it remains inevitably affected by statistical heterogeneity. Similarly to SGD, we prove that decreasing step-sizes do not help either: in fact, any schedule decreasing faster than $\Theta\left(1/t\right)$ leads to convergence to a constant value that depends on the initialization and the heterogeneity bound. 
Numerical results corroborate the theory, and deep learning experiments confirm its relevance for realistic settings.
\end{abstract}

\section{Introduction}

Modern deep learning applications demand intensive training on large amount of data, often distributed across decentralized silos or user personal devices.
To address such system constraints and comply with data regulations, learning algorithms have evolved towards more advanced and flexible systems that enable decentralized training at a global scale.
In such systems, not all workers participate at each training step, due to local faults, network issues or simply temporary unavailability. Moreover, they cannot usually exchange their data, either because of efficiency or privacy concerns. These are the main premises of Federated Learning (FL), a paradigm focused on privacy-preserving training from decentralized data.
Algorithms of this kind usually consist of an iterative two-step process involving 1) local training at client-side, each on its own private data, and 2) global optimization at the server, using aggregated local updates.
While this scheme promotes efficiency by looser synchronization, \textit{statistical heterogeneity} among clients' data and \textit{partial client participation} expose the optimization to \textit{client drift} and biased server updates.

Aiming for an effective solution to these problems, research has recently shifted towards extending momentum \cite{polyak1964heavyball} to distributed algorithms.
For example, a plethora of momentum-based FL algorithms have been proposed to overcome the adverse effects of data heterogeneity \cite{hsu2019measuring,reddi2020FedOpt,xu2021fedcm,Ozfatura2021FedADC,liu2023enhance,kim2024communication,cheng2024momentum,zaccone2025ghbm}.
Similarly, momentum is appealing in distributed learning to reduce the overall communication overhead \cite{wang2020SlowMo}, and recently has been scaled up to more decentralized environments \cite{douillard2024diloco}.
However, on a theoretical level, we only have a partial understanding of how momentum affects convergence in a decentralized regimen. \cite{cheng2024momentum} proved that momentum can converge under unbounded heterogeneity when all clients participate at each round (\textit{full participation}).
\cite{zaccone2025ghbm} went a step further, proposing a novel Generalized Heavy-Ball Momentum (\ghb{}) formulation that achieves the same convergence guarantees but with a more general \textit{cyclic partial participation} assumption. Yet, it is unclear whether the same result can be further extended to classical momentum under the same cyclic partial participation assumption and without bounded heterogeneity.

This work provides a clear answer to this question: \textit{can (classical) momentum enable convergence under unbounded heterogeneity in decentralized settings with partial participation?}

The answer is negative: even with (classical) momentum, the convergence rate relies on the heterogeneity bound.
This further confirms that \ghb{} \cite{zaccone2025ghbm} is, to the best of our knowledge, the only momentum-based distributed algorithm circumventing this limitation.

\noindent \textbf{Contributions.} We summarize our main results below. 
\begin{itemize}[noitemsep,nolistsep,leftmargin=*]
    \item We formally prove that, under cyclic sampling of clients, momentum does not eliminate the effect of data heterogeneity - a well recognized problem in decentralized and federated learning. 
    \item We further consider decreasing step-sizes, revealing that any schedule decaying faster than $\Theta(1/t)$ leads to convergence to a constant depending on the initialization and on the heterogeneity bound. 
    \item We validate the theory with numerical results on our theoretical problem, and extend the experimentation to deep learning problems, showing the relevance of our findings for realistic scenarios.
\end{itemize}
\IEEEpubidadjcol{}
\section{Related Works}
\label{sec:related}
Gradient Descent (GD) and its variants  have long been objective of study in the context of finite-sum optimization problems.
Restricting the gradient calculation to single function components (\ie{} a small subset of data) at each iteration, those methods trade off noisy updates for computational efficiency.
Most of the analyses address SGD or shuffling gradient methods \cite{safran2020howgood, koloskova2024on,liu2024onthelastiterate}.
\cite{Jentzen2020lower} provides sharp lower bounds on SGD for decreasing step-sizes, while \cite{nhuyen2019tight} prove dimension-independent lower bounds over all possible sequences of diminishing step-sizes. The recent work of \cite{kim2025incremental} studies the convergence rate of IGD at small iteration count.

While in all cases an heterogeneity bound is necessary, the above works consider algorithms \textit{without} momentum.
Since it has been proved that momentum has a variance reduction effect \cite{liu2020SGDM}, it is not clear i) if the fundamental reliance on the heterogeneity remains even with momentum, and ii) if decreasing step-sizes play a role.
In this work we analyze the simplest setting in which momentum could intuitively bring an advantage \wrt{} heterogeneous objectives: as we show, this corresponds to an instance of the IGD algorithm \textit{with} momentum.

\section{The Effect of Heterogeneity on Momentum}
We study the effect of momentum in heterogeneous settings by considering a minimal setup with two heterogeneous clients.
Our analysis is based on modeling the algorithm dynamics as a discrete-time linear system, and it reveals a clear decomposition: the \textit{zero-input response} captures objectives shared by all clients, while the \textit{zero-state response} isolates heterogeneous ones.
This formulation unveils the source of convergence limitations and the role of heterogeneity in the system’s behavior.
\subsection{Preliminaries}
\paragraph{Notation}
We use $T \in \mathbb{N}$ to denote total number of iterations of the algorithms, with $[T]$ representing the set $\{1,2,...,T\}$ and $t\in [T]$ the $t$-th iteration.
We denote as $f(\theta)$ the objective function parametrized by model parameters $\theta \in \mathbb{R}^d$, where $d$ is the dimensionality of the model.
We indicate with $\S{}$ the set of all clients and with $\S{t} \subset \S{}$ the ones active at $t$-th iteration. 
Throughout the paper, to express the asymptotic growth rate of the convergence rates, we use $\mathcal{O}, \Theta$ and $\Omega$ to respectively indicate an upper, exact bound and a lower bound, with symbols hiding constant factors.
\paragraph{Setting.}

We consider a distributed learning system where a set $\S{}$ of clients collaboratively solve a learning problem.
This can be formalized
as a finite-sum optimization problem, where an objective function $f(\theta)$ is expressed in terms of function components $f_i(\theta)$, with each client optimizing a different component.
Formally, the objective of the algorithm is finding:
\begin{equation}
    \label{eq:obj_fn}
    \theta^* =\arg\min_{\theta \in \mathbb{R}^d}
    \left[f(\theta):=\frac{1}{|\S{}|}\sum_{i\in \mathcal{S}} f_i(\theta) \right]
\end{equation}

\paragraph{Gradient-Based Methods with Momentum}
In modern deep learning applications, permutation-based variants of gradient descent (GD) are the most common algorithms. 
They reduce the computational burden by sampling and calculating a gradient over a function component $f_i$ at each step, mainly differing by the strategy used to select the component.
Among those, Stochastic Gradient Descent (SGD) and Incremental Gradient Descent (IGD) are most popular:
SGD samples $f_i$ uniformly and randomly, while IGD fixes any permutation of function components and samples cyclically from it. 
In this context, momentum has been used as a mechanism to reduce the impact of \textit{noise} introduced by sampling, and improve convergence.
Momentum consists in a moving average of past gradients, and it is often regarded as a way to reduce the variance of model updates \cite{liu2020SGDM}. Formally, the update rule of GD variants with momentum in its \textit{heavy-ball} form can be written  as:
\begin{align}
    \label{eq:mom_hb}
    \begin{alignedat}{2}
        {m}^t &\leftarrow (\theta^{t-1}-\theta^{t-2}), \\
        \theta^{t} &\leftarrow \theta^{t-1} - \eta(1-\beta) \nabla f^t(\theta^{t-1}) + \beta{m}^t
    \end{alignedat}
\end{align}    
where $\eta$ is the step-size, $\beta \in [0,1)$ is the momentum factor and $f^t(\theta)$ is the function objective component at time $t$.

\paragraph{From Centralized to Decentralized Algorithms}
In the context of decentralized and federated learning, clients often represent function components. This analogy is rooted in the fact that data among clients are expected to differ.
At each round $t \in [T]$, a fraction of $C \in (0,1]$ clients $\S{t}$ is selected for training. These clients may take gradients over mini-batch of data or additionally run an optimization algorithm locally over multiple local steps, and send back aggregated updates.

There are two main strategies to extend momentum to decentralized algorithms that use local steps. The first one, implemented by \fedavgm{} \cite{hsu2019measuring}, uses GD or SGD at the client side during local steps, and employs momentum at the server side, treating client's updates as a pseudo-gradient \cite{reddi2020FedOpt}.
The second, adopted by \fedcm{} \cite{xu2021fedcm} moves the same momentum term to the local optimization, and employs simple SGD at the server side. 
The rationale behind this choice is to provide better client drift correction during local optimization by employing a momentum calculated with server statistics, \ie{} averaged updates of clients sampled at each round, and frozen across the local steps.
As is, this is different than just momentum as local optimizer, since in that case it would be calculated and updated locally, failing at capturing other clients' contributions.

In this work we formalize the behavior of both algorithms assuming clients take \textit{full-batch gradients} (\ie{} no additional intra-client variance due to mini-batch sampling). 
Upon showing that the addition of local steps does not change the fundamental property of the algorithms, we then analyze the convergence behavior for the case of one local step.
\vspace{-3mm}

\subsection{Assumptions}
\label{sec:assumptions}
We assume objective functions are $\mu$-strongly convex, with clients sampled in a cyclic order (\cref{assum:str_convex,assum:cyclic_part}). Heterogeneity is captured by a bound on gradient dissimilarity between local and global objectives (\cref{assum:bounded_gd}), and we study how the convergence rate depends on it.
\vspace{1mm}
\begin{assumption}[Strong Convexity]
Let it be a constant $\mu>0$, then for any $i$, $\theta_1$, $\theta_2$ the following holds:
    \label{assum:str_convex}
    \begin{align*}
    f_i(\theta_2) &\geq f_i(\theta_1) + \langle\nabla f_i(\theta_1), \theta_2-\theta_1\rangle + \frac{\mu}{2}\normsq{\theta_2-\theta_1}
    \end{align*}
    Moreover, this implies that $f(\theta)$ is also $\mu$-strongly convex.
\end{assumption}

\begin{assumption}[Bounded Gradient Dissimilarity]
There exist a constant $G\geq0$ such that, $\forall i,\, \theta$: %
    \label{assum:bounded_gd}
    \begin{equation*}
        \frac{1}{|\S{}|} \sum_{i=1}^{|\S{}|} \left \| \nabla f_i(\theta) - \nabla f(\theta) \right \| \leq G
    \end{equation*}
\end{assumption}

\begin{assumption}[Cyclic Participation]
    \label{assum:cyclic_part}
    Let $\S{t}$ be the set of clients sampled at any round $t$. A sampling strategy is \textit{``cyclic``} with period $p=\nicefrac{1}{C}$ if:
    \begin{align*}
     \S{t} &\equiv \S{t-p} &\forall\; t>p \quad\land\quad
     \S{k} \cap \S{t} &= \varnothing &\forall\; k \in (t-p, t)
    \end{align*}
\end{assumption}

\subsection{Learning Problem Construction}
\label{section:lb_construction}
The intuition suggesting the use of momentum in a decentralized setting is that, being a moving average of past gradients, momentum achieves variance reduction effects \cite{liu2020SGDM}.
We construct a learning problem which should be favorable to momentum under partial client participation.
The global objective function is composed by only two objectives selected cyclically, one each round.
Both the minimum number of objective function components and the cyclic sampling are supposed to represent the easiest scenario for momentum, since we can guarantee that we observe the global objective every two rounds, ensuring momentum does not get biased towards either of the components.\\

\begin{lemma}[Behavior of \fedavgm{} and \fedcm{} on two one-dimensional clients]
\label{lemma:lb_construction}
    For any positive constants $G, \mu$, define $\mu$-strongly convex functions $f_1(\theta) := \frac{\mu}{2} \theta^2 + G\theta$ and $f_2(\theta):=\frac{\mu}{2} \theta^2 - G\theta$ satisfying assumption \ref{assum:bounded_gd} and such that $f(\theta)=\frac{1}{2}\left(f_1(\theta) + f_2(\theta)\right)$. Under cyclic participation (assumption \ref{assum:cyclic_part}) with $C=0.5$, for any $t \ge 1$ the evolution of \fedavgm{} and \fedcm{}, with global and local step-sizes $\eta_t, \eta_l$ and momentum weight $\beta$ is described by a discrete-time linear system with state-space representation:
    \begin{align*}
    \left\{\begin{alignedat}{2}
        \mathbf{z}[t] &= \Psi(t,1) \mathbf{z}[1] + \sum_{k=2}^{t}\Psi(t,k)\mathbf{B}\mathbf{u}[k] \\
        \mathbf{y}[t] &= \mathbf{C}\Psi(t,1) \mathbf{z}[1]\vphantom{\sum_{k=2}^{t}}
        + 
        \mathbf{C}\sum_{k=2}^{t}\Psi(t,k)\mathbf{B}\mathbf{u}[k]
    \end{alignedat}\right.
    \end{align*}
    where, given algorithm-dependent coefficients $p_t^{(a)}, q_t^{(a)}, r_t^{(a)}$:
    \begin{align*}
        & \mathbf{z}[t] =
            \begin{pmatrix}
                \theta^t &
                \theta^{t-1}
            \end{pmatrix}^\top,
        &\mathbf{u}[t] &=
        \begin{pmatrix}
            (-1)^t q_t^{(a)} G 
        \end{pmatrix}, \\
        & \mathbf{B} = 
            \begin{pmatrix}
                1 &
                0
            \end{pmatrix}^\top,
        & \mathbf{C} &=
            \begin{pmatrix}
                1 &
                0
            \end{pmatrix}\\
        & \mathbf{A}[t]=
            \begin{pmatrix}
                p_t^{(a)} & -r_t^{(a)} \\
                1                       & 0
            \end{pmatrix}
        &\Psi(t,k)&:=\prod_{s=k+1}^t\mathbf{A}[s]
    \end{align*}
\end{lemma}

\begin{proof}[Proof sketch]
    Assuming at each round the algorithm optimizes an objective function $f^t(\theta)$ cyclically (\eg{} $f_1(\theta)$ at $t$ odd and $f_2(\theta)$ otherwise), by \cref{lemma:update_fedavgm,lemma:update_fedcm} a generic update for \fedavgm{} and \fedcm{} can be written as:
    \begin{align}
    \label{eq:generic_mom_update}
    \theta^t \leftarrow p_t^{(a)}\theta^{t-1}+q_t^{(a)}(-1)^tG - r_t^{(a)}\theta^{t-2} 
    \end{align}
    \cref{eq:generic_mom_update} can analyzed as a discrete-time linear system, where $\mathbf{z}[t]$, $\mathbf{u}[t]$ and $\mathbf{y}[t]=\theta^t$ are respectively the state, the input and the output at $t$-th round.
    The evolution of the system has the following state-space form:
    \begin{align}
    \left\{\begin{alignedat}{2}
        \mathbf{z}[t] &= \mathbf{A}[t]\mathbf{z}[t-1] + \mathbf{B}\mathbf{u}[t] \\
        \mathbf{y}[t] &= \mathbf{C}\mathbf{z}[t] 
    \end{alignedat}\right.
    \end{align}
    Unrolling the recursion leads to the lemma statement. Complete proof is presented in \cref{proof:lemma:lb_construction}.
\end{proof}
To draw intuition about the behavior of the system, let us express $f_{1,2}(\theta)=f_{hom}(\theta)+f_{het}(\theta)$, where $f_{hom}(\theta)=f(\theta)$ and $f_{het}(\theta)=\pm G\theta$. Then, the update \wrt{} the shared objective maps onto the \textit{natural response} of the system, while $\nabla f_{het}(\theta)$ appears as external force of the system. In practice, at each iteration $t$, the presence of heterogeneity 
acts as ``disturb signal" to the optimization of the global objective.
This offers immediate understanding of the impact of noise on the algorithms' convergence:
if there was no input (\ie{} $f_{1,2}(\theta)$ were homogeneous or both sampled at each round), then convergence would depend only on initial conditions, with an exponentially fast rate under constant $p_t^{(a)}$ (\ie{} with proper constant step-size $\eta_t=\eta$).
Conversely, the presence of heterogeneity leads to a convergence rate determined by how the terms related to the initial conditions (called \textit{zero-input response}) and the input (called \textit{zero-state response}) interact: this depends on the choice of step-size $\eta_t$, which enters both in the term $p_t^{(a)}$ and and as scaling to input in $q_t^{(a)}$.

These terms also depend on the number of local steps, and can be calculated by unrolling the one round server update of the algorithms. For \fedavgm{}, by \cref{lemma:update_fedavgm}, we have that:
\begin{align}
    &\left\{
    \begin{alignedat}{3}
        &p_t^{(\fedavgm{})} = 1+\beta+\tilde{\eta}_t((1-\mu\eta_l)^J-1) \\
        &q_t^{(\fedavgm{})} = \tilde{\eta}_t\eta_l\sum_{j=0}^{J-1}(1-\mu\eta_l)^j \\
        &r_t^{(\fedavgm{})} = \beta 
    \end{alignedat}
    \right.
\intertext{Conversely, for \fedcm{}, by \cref{lemma:update_fedcm}, we have that:}
    &\left\{
    \begin{alignedat}{3}
        &p_t^{(\textsc{CM})} = 1+\frac{\beta}{J}\sum_{j=0}^{J-1}(1-\mu\tilde{\eta}_l)^j +{\eta}_t((1-\mu\tilde{\eta}_l)^J-1) \\
        &q_t^{(\textsc{CM})} = {\eta}_t\tilde{\eta}_l\sum_{j=0}^{J-1}(1-\mu\tilde{\eta}_l)^j \\
        &r_t^{(\textsc{CM})} = \frac{\beta}{J}\sum_{j=0}^{J-1}(1-\mu\tilde{\eta}_l)^j 
    \end{alignedat}
    \right.
\end{align}
where $\eta_t,\eta_l$ are global and local step-sizes, $\tilde{\eta}_t=\eta_t(1-\beta)$ and $\tilde{\eta}_l=\eta_l(1-\beta)$.
Comparing the $q_t^{(a)}$ terms, we notice that:
\begin{equation}
    \label{eq:qt_comparison}
    q_t^{(\textsc{CM})} \geq q_t^{(\fedavgm{})} \geq \eta_t\eta_l(1-\beta)
\end{equation}
with the equality holding when $J=1$. This means that: (i) for both algorithms, increasing the number of local steps worsens the dependency on the heterogeneity bound and (ii) \fedcm{} has a worse dependency than \fedavgm{}.
The additional multiplicative factors due to local steps represent the effect of the so-called \textit{client drift}, that is the phenomenon by which heterogeneous clients ``drift away'' from the ideal global update. These factors converge for $J\to \infty$, so the $q_t^{(a)}$ terms differ at most by a constant from the case when only one local step is taken.

Determining if one of the two algorithms is apriori better at optimizing the zero-input response would involve explicit calculation of the maximal eigenvalue for both algorithms, which depend both on $p_t^{(a)}, r_t^{(a)}$.
From numerical calculations, it is possible to show that the algorithm with lowest maximal eigenvalue, \ie{} the faster, depends on the hyperparameters.
However, since $p_t^{(a)}, r_t^{(a)}$ only differ by a constant
the  asymptotic convergence rate remains the same for any $J$.
Having verified that local steps do not provide convergence advantages at the same amount of iterations, in the rest of this work we analyze the case of only one local step size. In this case, by \cref{corollary:fedavgm_fedcm}, \fedavgm{} and \fedcm{} both correspond to an instance of IGD with momentum, and share the same update rule leading to state and input matrices as:
\begin{align}
    &\mathbf{A}[t]=
        \begin{pmatrix}
            1+\beta-\mu\tilde{\eta}_t & -\beta \\
            1                       & 0
        \end{pmatrix},
    &\mathbf{u}[t]=
        \begin{pmatrix}
            (-1)^t \tilde{\eta}_tG
        \end{pmatrix}
\end{align}

\subsection{Convergence under Constant Step-sizes}
The following theorem reveals that, similarly as it is known for vanilla SGD and IGD, momentum does not bring any asymptotic advantage in the convergence rate. This directly implies that algorithms based on classical momentum \textit{cannot} provide strong theoretical guarantees against statistical heterogeneity in decentralized and federated learning settings under partial participation (\ie{} an heterogeneity bound $G$ is still necessary).
Since our analysis does not take into account local steps, the following theorem holds for both \fedavgm{} and \fedcm{} under cyclic client participation.\\

\begin{theorem}
\label{thm:lb_constant_step}
    For any positive constants $G, \mu$ there exist $\mu$-strongly convex functions satisfying assumption \ref{assum:bounded_gd} for which, under proper constant step-size $\eta$ and for any momentum factor $\beta \in [0,1)$, the output of {\fedcm} and {\fedavgm} under cyclic partial participation (assumption \ref{assum:cyclic_part}), has the following asymptotic error:
\begin{equation*}
    f(\theta^t) - f(\theta^*) = \Theta \left(  \frac{G^2}{\mu T^2} \right)
\end{equation*}
\end{theorem}
\begin{proof}[Proof sketch]
    Starting from the representation in \cref{lemma:lb_construction}, the output of the system can be analyzed studying separately the zero-input and zero-state response, and results combined thanks to linearity.
    Under constant step size, $\mathbf{A}[t]=\mathbf{A}\, \forall t \in [T]$, implying $\Psi(t,1)=\mathbf{A}^{t-1}$, so the (exponential) convergence of the zero-input response reduces to imposing the eigenvalues of $\mathbf{A}$ to be strictly less than one, leading to:
    \begin{equation}
    \label{eq:proof_sketch_constat_step:eta_bound}
        \eta \in \left( 0, \frac{2(1+\beta)}{\mu(1-\beta)}\right), \qquad \beta \in [0,1)
    \end{equation}
    The solution of the zero-state response follows from noticing that the input $\mathbf{u}[t]$ is $2$-periodic and non-vanishing, leading to a limit cycle of the same period:
    \begin{equation}
        \theta^t = (-1)^t \frac{\eta(1-\beta)G}{2(1+\beta)-\mu\eta(1-\beta)}
    \end{equation}
    Notice that the amplitude of the cycle limit monotonically increases with $\eta$. As is, to control the error at convergence, we must impose a small step size as:
    \begin{equation}
        \frac{c_1}{\mu T}\frac{1+\beta}{1-\beta} < \eta < \frac{c_2}{\mu T}\frac{1+\beta}{1-\beta}
    \end{equation}
    with $0<c_1<c_2\leq2$ and $T>1$. 
    The result follows from imposing the above constraint, and noting that $f(\theta)=\frac{\mu}{2}\theta^2$.
    Complete proof is presented in \cref{proof:thm:lb_constant_step}.
\end{proof}
The fundamental insight of \cref{thm:lb_constant_step} regards the asymptotic behavior of the zero-input and zero-state response depending on the choice of the step size.
In fact, the zero-input response converges exponentially fast (the faster the higher the step-size following \cref{eq:proof_sketch_constat_step:eta_bound}), matching the convergence rate of GD.
Conversely, since the zero-state response converges to a 2-period cycle limit 
, a step-size as small as $\eta =\Theta(1/T)$ must be imposed to obtain a linear rate.

\subsection{Convergence under Decreasing Step-sizes}
The intuitive reason for adopting decreasing step-sizes lies on the observation that heterogeneity enters the optimization as an external input, scaled by the effective step-size $\eta_t(1-\beta)$. This suggest that decreasing $\eta_t$ over time may offer a benefit not visible when it is kept constant.
We study the problem in \cref{lemma:lb_construction} under a polynomial decreasing step-size schedule of the type $\eta_t=\eta/t^\alpha$, where $t$ is the current iteration and $\alpha>0$ is an hyperparameter controlling the decay rate of $\eta_t$.
We show that, even when $\eta_t$ is decreasing, the dependence on the heterogeneity bound cannot be eliminated, and that overly fast-decaying step-size schedules are detrimental. \\

\begin{theorem}
\label{thm:lb_decreasing_step}
    For any positive constants $G, \mu$, $\alpha$ there exist $\mu$-strongly convex functions satisfying assumption \ref{assum:bounded_gd} for which, under decreasing step-size $\eta_t \sim \mathcal{O}\left(1/t^\alpha \right)$,the output of {\fedcm} and {\fedavgm} under cyclic participation (assumption \ref{assum:cyclic_part}), even assuming initialization at optimum ($\theta^0=\theta^*$), has the following error:
    \begin{align*}
    f(\theta^t) - f(\theta^*) &=
        \left\{\begin{alignedat}{2}
             & \Theta \left(  \frac{G^2}{\mu t^{2\alpha}} \right) &\quad& \text{if } 0 < \alpha < 1  \\
             & \Theta \left(  \frac{G^2}{\mu t^{2\min(\mu\eta,1)}} \right) &\quad& \text{if } \alpha = 1  \\
             & \Theta \left(
             \frac{G^2}{\mu} \right)
             &\quad& \text{if }  \alpha > 1
        \end{alignedat}\right.
    \end{align*}
\end{theorem}
\begin{proof}[Proof sketch]
    The main difficulty \wrt{} the analysis in \cref{thm:lb_constant_step} is that the underlying system is not time-invariant anymore. Indeed, since one of the eigenvalues of the $\mathbf{A}$ matrix depends on time and tends to $1$, it is not possible to rely on the analysis of eigenvalues.
    To study the system, we express decompose the $\mathbf{A}[t]$ matrix as:
    \begin{equation}
    \mathbf{A}[t]=
        \underbrace{\begin{pmatrix}
        1+\beta & -\beta \\
            1                       & 0
        \end{pmatrix}}_{:=\mathbf{A}^\infty} + 
        \underbrace{\begin{pmatrix}
            -\frac{\mu\tilde{\eta}}{t^\alpha} & 0 \\
            0 & 0
        \end{pmatrix}}_{:=\mathbf{E}[t]}
    \end{equation}
    and then diagonalize it \wrt{} the $\mathbf{A}^\infty$ matrix, to decouple the directions corresponding to the marginally stable eigenvalue $\lambda_1=1$ and the asymptotic stable $\lambda_2=\beta < 1$:
    \begin{equation}
        \bar{\mathbf{z}}[t] =(\mathbf{\Lambda} + \mathbf{H}[t])\bar{\mathbf{z}}[t-1] +\mathbf{W}\mathbf{u}[t], \quad 
        \mathbf{\Lambda} =
        \begin{pmatrix}
            1 & 0 \\
            0 & \beta
        \end{pmatrix}
    \end{equation}
    Since the system is not diagonal, each component $\bar{z}_{1,2}[t]$, enters the other as external input, \eg{} for $\bar{z}_2[t]$:
    \begin{align*}
        \bar{z}_2[t] &= 
        \underbrace{\vphantom{\sum_{s=2}^{t}}\Psi_2(t,1,\alpha)\bar{z}_2[1]}_{\text{zero-input response}}
        + 
        \underbrace{\sum_{s=2}^{t}\Psi_2(t,s,\alpha)\frac{\bar{z}_1[s-1]}{s^\alpha}}_{\text{coupling term with $\bar{z}_1$}} +\\
        &-
        \underbrace{\sum_{s=2}^{t}\Psi_2(t,s,\alpha) \frac{\eta G}{s^\alpha}}_{\text{zero-state response}}%
    \end{align*}
    and similarly for $\bar{z}_1[t]$. The analysis proceeds by finding a set of self-consistent hypotheses for the asymptotic behavior of $\bar{z}_{1,2}[t]$. Following this, we obtain that $\bar{z}_2[t]=\Theta\left(\eta G / t^\alpha\right),\, \forall \alpha>0$ independent of $\bar{z}_1[t]$: this is expected, since $\bar{z}_2[t]$ is the direction associated with $\lambda_2=\beta$. Conversely:
    \begin{align}
        \bar{z}_1[t]=
        \left\{
        \begin{alignedat}{3}
            &\Theta\left(\frac{G}{\mu t^\alpha}\right) &\,& \text{ if }0 < \alpha <1 \\
            &\Theta\left(\frac{G}{\mu t^{\mu\eta}}\right) &\,& \text{ if } \alpha =1 \\
            &\Theta\left(\frac{G}{\mu}\right) &\,& \text{ if } \alpha > 1 \\
        \end{alignedat}
        \right.
    \end{align}
    The convergence rate follows from noticing that $y[t]=z_1[t]=\bar{z}_1[t]+\beta\bar{z}_2[t]$ and analyzing the dominance of each term varying $\alpha,\eta$. Noting that $f(\theta)=\frac{\mu}{2}\theta^2$ leads to the theorem result.
    Complete proof is deferred to \cref{proof:thm:lower_bound_decreasin_stepsize}.
\end{proof}

\paragraph{Slowly-decreasing step-sizes}
 When the decay rate of the step-size is sufficiently slow (\ie{} $0<\alpha <1$), the convergence rate is strictly slower than in \cref{thm:lb_constant_step}, as $2\alpha<2$, and the dependence on the heterogeneity bound $G$ remains.
 From the mathematical point of view, the bottleneck in the rate arises from the solution of the zero-input response, which decays as a polynomial in $\alpha$, while the zero-state response still decays exponentially fast. As such, for large $t$ the rate is dominated by the former term, and the final convergence value $\theta^t$ is the same irrespective of initial conditions $\theta^0$.

\paragraph{Fast-decreasing step-sizes}

When $\alpha=1$, the convergence rate depends on the choice of initial step size $\eta$. When a small $\eta<1/\mu$ is chosen, the rate depends on $\mu\eta$, getting slower as $\eta$ is chosen smaller. On the other hand, when a large $\eta\geq1/\mu$ is chosen, the rate matches the one in \cref{thm:lb_constant_step}. Similar findings have been observed for SGD under the same step-size schedule by \cite{Jentzen2020lower}.
Mathematically, the transition between $t^{-\mu\eta}$ to $t^{-1}$ in the rate arises 
because
the state transition matrix $\Psi(t,s)$ now decays only polynomially to zero, not exponentially as in the previous case. As is, the rate now depends on how the zero-input and zero-state responses interact: when $\eta < 1/\mu$, a term depending on the initialization affects the rate, so $\theta^t$ will depend on $\theta^0$. On the contrary, when $\eta > 1/\mu$, the rate is dominated only by the response to heterogeneity.

\paragraph{Overly fast-decreasing step-sizes}
When the step-size decays faster than linearly, the algorithm fails to reach an arbitrarily small optimality gap.
Both the solutions of the homogeneous and heterogeneous part of the system in \cref{lemma:lb_construction} are affected, because the state transition matrix does not longer decay to zero.
This means that, not only the zero-state response converges to a constant depending on $G$, but also the the zero-input response converges to a constant depending on the initialization.

\subsection{Circumventing the Momentum Lower Bounds}
The findings in this section confirm
classical momentum cannot be employed in decentralized learning to completely overcome the effects of statistical heterogeneity.
To the best of authors' knowledge, the only momentum-based algorithm circumventing this limitation is the Generalized Heavy-Ball Momentum (\ghb{}) \cite{zaccone2025ghbm}. 
As authors explain, leveraging an incremental aggregated gradient perspective, its momentum update rule approximates the one classical momentum has in full participation. 
Modeled in such a way, the heterogeneous term term arising in \cref{lemma:lb_construction} does not appear even in (cyclic) partial participation, recovering the rate of classical momentum in full participation.
Therefore, the limitations we refer to in this paper do not apply to \ghb{}.

\section{Numerical Results}
\label{sec:experiments}

\subsection{Theoretical Experiments on \texorpdfstring{$\mu$}{mu}-strongly Convex Functions}
\begin{table*}[t]
\centering
\caption{\textbf{Effect function heterogeneity and decreasing step-size on \textsc{IGD} with (left) and without momentum (right):} $\theta^t$ after $T = 10^6$ iterations for the problem in \cref{lemma:lb_construction}. Heterogeneity affects convergence linearly, and step-size schedules decaying faster than $\Theta\left(1/t^\alpha\right)$ lead to worse solutions, both when not starting at the optimum (\ie{} $\theta^0 \neq 0$ and $G=0$) and when objectives are heterogeneous (\ie{} $G > 0$ and $\theta^0 = \theta^*$).}
\label{tab:theory_exp}
\resizebox{0.99\textwidth}{!}{%
\begin{tabular}{@{}lrrrrr|rrrrrr@{}}
\toprule
\multicolumn{1}{c}{\multirow{2}{*}{\sc Step-size schedule}} &
  \multicolumn{2}{c}{$G=100$} &
  \multicolumn{2}{c}{$G=10$} &
  \multicolumn{1}{c}{$G=0$} &
  \multicolumn{2}{c}{$G=100$} &
  \multicolumn{2}{c}{$G=10$} &
  \multicolumn{1}{c}{$G=0$} \\ 
\cmidrule(l){2-3} \cmidrule(l){4-5} \cmidrule(l){6-6} \cmidrule(l){7-8} \cmidrule(l){9-10} \cmidrule(l){11-11}
\multicolumn{1}{c}{} &
  \multicolumn{1}{c}{$\theta^0 = 0$} &
  $\theta^0 = 10$ &
  \multicolumn{1}{c}{$\theta^0 = 0$} &
  $\theta^0 = 10$ &
  \multicolumn{1}{c}{$\theta^0 = 10$} &
  \multicolumn{1}{c}{$\theta^0 = 0$} &
  $\theta^0 = 10$ &
  \multicolumn{1}{c}{$\theta^0 = 0$} &
  $\theta^0 = 10$ &
  \multicolumn{1}{c}{$\theta^0 = 10$} \\
\midrule
\textsc{Constant}                  & $2.5\mathrm{e}{-05}$
 & $2.5\mathrm{e}{-05}$ & $2.5\mathrm{e}{-06}$ & $2.5\mathrm{e}{-06}$  & $5.7\mathrm{e}{-08}$ & $1.5\mathrm{e}{-05}$ & $3.7\mathrm{e}{+00}$ & $1.5\mathrm{e}{-06}$ & $3.7\mathrm{e}{+00}$ & $3.7\mathrm{e}{+00}$ \\ \cmidrule(r){1-1}
\makecell[l]{\textsc{Polynomial}}                     &  &  &  &  &  &  &  &  &  &  \\
\quad $\alpha = 0.1$              & $7.2\mathrm{e}{+00}$ & $7.2\mathrm{e}{+00}$ & $7.2\mathrm{e}{-01}$ & $7.2\mathrm{e}{-01}$ & $-5.0\mathrm{e}{-324}$ & $7.2\mathrm{e}{+00}$ & $7.2\mathrm{e}{+00}$ & $7.2\mathrm{e}{-01}$  & $7.0\mathrm{e}{-01}$ & $5.0\mathrm{e}{-324}$ \\
\quad $\alpha = 0.5$              & $2.5\mathrm{e}{-02}$ & $2.5\mathrm{e}{-02}$ & $2.5\mathrm{e}{-03}$ & $2.5\mathrm{e}{-03}$ & $1.5\mathrm{e}{-323}$ & $2.5\mathrm{e}{-02}$ & $2.5\mathrm{e}{-02}$ & $2.5\mathrm{e}{-03}$ & $2.5\mathrm{e}{-03}$ & $9.4\mathrm{e}{-322}$ \\
\quad $\alpha = 1$                & $2.5\mathrm{e}{-05}$ & $2.5\mathrm{e}{-05}$ & $2.5\mathrm{e}{-06}$ & $2.5\mathrm{e}{-06}$ & $-1.7\mathrm{e}{-78}$ & $-8.6\mathrm{e}{-06}$ & $-8.3\mathrm{e}{-06}$ & $-8.6\mathrm{e}{-07}$ & $-6.0\mathrm{e}{-07}$ & $2.6\mathrm{e}{-07}$ \\
\rowcolor{red!30}
\quad $\alpha = 2$                & $4.8\mathrm{e}{+01}$ & $5.7\mathrm{e}{+01}$ & $4.8\mathrm{e}{+00}$ & $1.4\mathrm{e}{+01}$ & $9.0\mathrm{e}{+00}$ & $-1.9\mathrm{e}{+01}$ & $-1.9\mathrm{e}{+01}$ & $-1.9\mathrm{e}{+00}$ & $-1.8\mathrm{e}{+00}$ & $1.0\mathrm{e}{-01}$ \\ \cmidrule(r){1-1}
\makecell[l]{\textsc{Exponential}}                   &  &  &  &  &  &  &  &  &  &  \\
\quad $\gamma = 0.9999$           &  $1.5\mathrm{e}{-17}$ & $1.5\mathrm{e}{-17}$  & $2.4\mathrm{e}{-18}$ & $2.4\mathrm{e}{-18}$ & $-1.5\mathrm{e}{-323}$ & $1.9\mathrm{e}{-17}$ & $1.9\mathrm{e}{-17}$ & $2.3\mathrm{e}{-18}$ & $2.3\mathrm{e}{-18}$ & $0.0\mathrm{e}{+00}$ \\

\quad $\gamma = 0.999$            &  $1.8\mathrm{e}{-17}$ & $1.8\mathrm{e}{-17}$ & $-6.3\mathrm{e}{-18}$ & $-6.3\mathrm{e}{-18}$ & $-1.7\mathrm{e}{-163}$ & $2.3\mathrm{e}{-17}$ & $2.3\mathrm{e}{-17}$ & $7.0\mathrm{e}{-18}$ & $7.0\mathrm{e}{-18}$ & $0.0\mathrm{e}{+00}$ \\
\quad $\gamma = 0.99$             &  $1.1\mathrm{e}{-14}$ & $9.8\mathrm{e}{-15}$ & $1.1\mathrm{e}{-15}$ & $-3.5\mathrm{e}{-16}$ & $-1.5\mathrm{e}{-15}$ & $1.5\mathrm{e}{-16}$ & $1.5\mathrm{e}{-16}$ & $1.8\mathrm{e}{-18}$ & $1.8\mathrm{e}{-18}$ & $1.1\mathrm{e}{-65}$ \\
\rowcolor{red!30}
\quad $\gamma = 0.9$              &  $-7.2\mathrm{e}{+00}$ & $-6.2\mathrm{e}{+00}$  & $-7.2\mathrm{e}{-01}$ & $-2.4\mathrm{e}{+01}$ & $9.5\mathrm{e}{-01}$ & $-1.0\mathrm{e}{-04}$ & $-8.2\mathrm{e}{-05}$ & $-1.0\mathrm{e}{-05}$ & $1.1\mathrm{e}{-05}$ & $2.2\mathrm{e}{-05}$ \\ 
\bottomrule
\end{tabular}%
}
\end{table*}

We provide numerical results confirm our theoretical findings, evaluating three step-size schedules: constant (as analyzed in \cref{thm:lb_constant_step}), polynomially decreasing ($\eta_t=\eta/t^\alpha$, with $\alpha>0$, as in \cref{thm:lb_decreasing_step}), and exponentially decreasing ($\eta_t=\eta \gamma^t$, with $\gamma \in (0,1)$). The experiments, shown in \cref{tab:theory_exp} for comparison with and without momentum, confirm that momentum is affected by heterogeneity, and that that fast-decaying schedules negatively affect convergence to the optimum.

\paragraph{Constant and Slowly-decreasing Step-sizes}
Results in \cref{tab:theory_exp} show that, when the learning rate is constant or slowly decreasing (\ie{} $\alpha<1$), the final value at convergence always linearly depends on the heterogeneity bound $G$, and it is irrespective of initialization. This validates the theory, which predicts an exponential decay rate of the initial conditions and a linear decay of the perturbation caused by heterogeneity.
The result of constant learning rate and linear decay ($\alpha=1$) are equal in all cases but when the system is homogeneous (\ie{} $G=0$): in this case, since the decay rate of the initial conditions is exponential, a bigger step-size is better 
This motivates why the smaller the decay, the closer the solution is to the optimum, which is contrary to the heterogeneous cases.

\begin{figure*}[!ht]
    \centering
    \includegraphics[ width=0.49\textwidth]{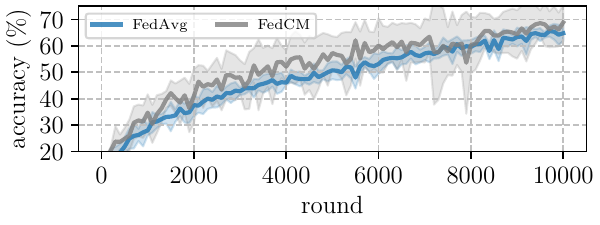}
    \includegraphics[ width=0.49\textwidth]{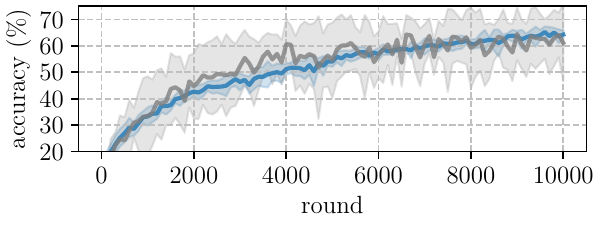}
    \vspace{-3mm}
    \caption{\textbf{\fedavg{} and \fedcm{} under cyclic participation:} under high heterogeneity and partial participation, FL-methods based on classical momentum do not offer a substantial improvement over simpler methods without momentum. Results on \cifar{10} with {\resnet} (left) and {\lenet} (right).
    The reference accuracy in centralized settings  is $\approx86\%$ for {\lenet} and $\approx 89\%$ for {\resnet}.
    }
    \label{fig:fl_exp}
\end{figure*}

\paragraph{Fast-decreasing Step-sizes}

\begin{table}[t]
\centering
\caption{\textbf{Impact of step-size with fast decaying polynomial schedule on \textsc{IGD} with momentum:} $\theta^t$ after $T = 10^6$ iterations for the learning problem in \cref{lemma:lb_construction}, with $\epsilon=10^{-2}$.
}
\label{tab:poly_alpha_1}
\resizebox{0.99\linewidth}{!}{%
\begin{tabular}{llrr}
\toprule
\multirow{2}{*}{\sc \thead{Polynomial \\ Decay Rate}} & \multicolumn{1}{c}{\multirow{2}{*}{\sc \thead{Initial \\ Step-Size}}} & \multicolumn{2}{c}{$G=10$} \\ \cmidrule(l){3-4}
         & \multicolumn{1}{c}{}                           & $\theta^0=0$ & $\theta^0=10$ \\
\midrule
$\alpha=1$ & $\eta=\frac{1(1+\beta)}{\mu (1-\beta)}-\epsilon$ &   $2.5\mathrm{e}{-06}$         &     $2.5\mathrm{e}{-06}$        \\
$\alpha=1$ & $\eta=\frac{1}{\mu} - \epsilon$                  &   $-3.9\mathrm{e}{-06}$        &    $-1.2\mathrm{e}{-04}$    \\
\bottomrule
\end{tabular}%
}
\end{table}
When $\alpha=1$, the decay rate of initialization and heterogeneity interact, as they are both polynomial, and the overall rate depends on the choice of the step-size. 
Indeed, as shown in \cref{tab:poly_alpha_1}, when $\eta > 1/\mu$ 
the solution depends on the heterogeneity, since the decay rate of the initialization is $\mathcal{O}(t^{-\mu\eta})$, which is faster than $\mathcal{O}(t^{-1})$: this makes the solution independent of $\theta^0$
On the contrary, $\eta<1/\mu$ the decay rate of the initialization is slower than $\mathcal{O}(t^{-1})$, so the final solution $\theta^t$ is different for $\theta^0=0$ and $\theta^0=10$.
When the step-size decay rate is too fast, the system does not converge to the optimum, but to a final value depending on initialization and heterogeneity, as highlighted by the red rows in \cref{tab:theory_exp}.

\subsection{Federated Learning Experiments}
\paragraph{Dataset and Models}
We use \cifar{10} with training images preprocessed by applying random crops, random horizontal flips and both train and test images finally normalized according to their mean and standard deviation.
As models, we used an architecture similar to \textsc{LeNet-5} as proposed in \cite{hsu2020FederatedVisual}, and a {\resnet} as described in \cite{he2015deep}, following the implementation provided in \cite{resnetCifarGithub}. Since batch normalization~\cite{ioffe2015batch} layers have been shown to hamper performance in learning from decentralized data with skewed label distribution \cite{pmlr-v119-hsieh20a}, we replaced them with group normalization \cite{Wu2018GroupNorm}.
\paragraph{Experimental Setting}
 we split the dataset among $|\mathcal{S}|=100$ clients following a common FL procedure proposed by \cite{hsu2020FederatedVisual}, and sample $C=10\%$ of them at each round.
We choose to simulate the most severe level of heterogeneity, since it has been showed to be a practical predictor of algorithms' performance with more complex architectures and large-scale datasets \cite{zaccone2025ghbm}.
Results are shown as average of $5$ independent runs, with standard deviation depicted with shaded areas.

\paragraph{Results}
\cref{fig:fl_exp} shows the test accuracy over training rounds of \fedavg{} and \fedcm{}. The experiments are conducted under cyclic participation, to reflect the setting analyzed in this paper and simulate a condition favorable to momentum.
The results clearly indicate that classical momentum is ineffective in high heterogeneous decentralized settings with partial participation. As in our theoretical experiments, this is motivated by the fact that momentum does not overcome the noise due to sampling only a subset of function components at each round. Similar experimental findings are reported for the case of random uniform client participation in \cite{zaccone2025ghbm}.

\section{Conclusions}

This paper addresses a gap in understanding the role of momentum in distributed optimization with statistical heterogeneity and partial worker participation. While momentum is appealing to build robustness to statistical heterogeneity, our work demonstrates that it does not inherently overcome the challenges posed by heterogeneous data.
By unveiling this fundamental limitation, this work 
provides a more realistic basis for its use in heterogeneous decentralized environments.

\bibliography{main}

@inproceedings{reddi2020FedOpt,
    title={Adaptive federated optimization},
    author={Reddi, Sashank and Charles, Zachary and Zaheer, Manzil and Garrett, Zachary and Rush, Keith and Kone{\v{c}}n{\`y}, Jakub and Kumar, Sanjiv and McMahan, H Brendan},
    booktitle={ICLR},
    year={2021}
}

@inproceedings{wang2020SlowMo,
    title={SlowMo: Improving Communication-Efficient Distributed SGD with Slow Momentum},
    author={Jianyu Wang and Vinayak Tantia and Nicolas Ballas and Michael Rabbat},
    booktitle={ICLR},
    year={2020}
}

@inproceedings{Ozfatura2021FedADC,
    author={Ozfatura, Emre and Ozfatura, Kerem and Gündüz, Deniz},
    booktitle={2021 IEEE International Symposium on Information Theory (ISIT)}, 
    title={FedADC: Accelerated Federated Learning with Drift Control}, 
    year={2021},
}

@article{hsu2019measuring,
    title={Measuring the Effects of Non-Identical Data Distribution for Federated Visual Classification}, 
    author={Tzu-Ming Harry Hsu and Hang Qi and Matthew Brown},
    year={2019},
    journal={arXiv preprint arXiv:1909.06335},
}

@InProceedings{hsu2020FederatedVisual,
    author="Hsu, Tzu-Ming Harry
    and Qi, Hang
    and Brown, Matthew",
    editor="Vedaldi, Andrea
    and Bischof, Horst
    and Brox, Thomas
    and Frahm, Jan-Michael",
    title="Federated Visual Classification with Real-World Data Distribution",
    booktitle="ECCV",
    year="2020",
}

@article{polyak1964heavyball,
    author = {Polyak, Boris},
    year = {1964},
    title = {Some methods of speeding up the convergence of iteration methods},
    journal = {Ussr Computational Mathematics and Mathematical Physics},
}

@inproceedings{liu2020SGDM,
    author = {Liu, Yanli and Gao, Yuan and Yin, Wotao},
    booktitle = {NeurIPS},
    title = {An Improved Analysis of Stochastic Gradient Descent with Momentum},
    year = {2020}
}

@article{he2015deep,
      title={Deep Residual Learning for Image Recognition}, 
      author={Kaiming He and Xiangyu Zhang and Shaoqing Ren and Jian Sun},
      year={2015},
      journal={arXiv preprint arXv:1512.03385},
}

@misc{resnetCifarGithub,
    author       = "Yerlan Idelbayev",
    title        = "Proper {ResNet} Implementation for {CIFAR10/CIFAR100} in {PyTorch}",
    year = "2021"
}

@InProceedings{Wu2018GroupNorm,
    author = {Wu, Yuxin and He, Kaiming},
    title = {Group Normalization},
    booktitle = {ECCV},
    year = {2018}
}

@InProceedings{pmlr-v119-hsieh20a,
    title = 	 {The Non-{IID} Data Quagmire of Decentralized Machine Learning},
    author =       {Hsieh, Kevin and Phanishayee, Amar and Mutlu, Onur and Gibbons, Phillip},
    booktitle = 	 {ICML},
    year = 	 {2020},
}

@article{xu2021fedcm,
    title={FedCM: Federated Learning with Client-level Momentum}, 
    author={Jing Xu and Sen Wang and Liwei Wang and Andrew Chi-Chih Yao},
    year={2021},
    journal={arXiv preprint arXiv:2106.10874}
}

@inproceedings{ioffe2015batch,
    title={Batch normalization: Accelerating deep network training by reducing internal covariate shift},
    author={Ioffe, Sergey and Szegedy, Christian},
    booktitle={ICML},
    year={2015},
}

@inproceedings{cheng2024momentum,
    title={Momentum Benefits Non-iid Federated Learning Simply and Provably},
    author={Ziheng Cheng and Xinmeng Huang and Pengfei Wu and Kun Yuan},
    booktitle={ICLR},
    year={2024},
}

@misc{liu2023enhance,
      title={Enhance Local Consistency in Federated Learning: A Multi-Step Inertial Momentum Approach}, 
      author={Yixing Liu and Yan Sun and Zhengtao Ding and Li Shen and Bo Liu and Dacheng Tao},
      year={2023},
      eprint={2302.05726},
      archivePrefix={arXiv},
      primaryClass={eess.SY},
      url={https://arxiv.org/abs/2302.05726}, 
}

@inproceedings{
douillard2024diloco,
title={DiLoCo: Distributed Low-Communication Training of Language Models},
author={Arthur Douillard and Qixuan Feng and Andrei Alex Rusu and Rachita Chhaparia and Yani Donchev and Adhiguna Kuncoro and MarcAurelio Ranzato and Arthur Szlam and Jiajun Shen},
booktitle={2nd Workshop on Advancing Neural Network Training: Computational Efficiency, Scalability, and Resource Optimization (WANT@ICML 2024)},
year={2024}
}

@article{
zaccone2025ghbm,
title={Communication-Efficient Heterogeneous Federated Learning with Generalized Heavy-Ball Momentum},
author={Riccardo Zaccone and Sai Praneeth Karimireddy and Carlo Masone and Marco Ciccone},
journal={Transactions on Machine Learning Research},
year={2025},
}

@inproceedings{kim2024communication,
    author = {Kim, Geeho and Kim, Jinkyu and Han, Bohyung},
    title = {Communication-Efficient Federated Learning with Accelerated Client Gradient},
    booktitle = {CVPR},
    year = {2024},
 }

@inproceedings{
kim2025incremental,
title={Incremental Gradient Descent with Small Epoch Counts is Surprisingly Slow on Ill-Conditioned Problems},
author={Yujun Kim and Jaeyoung Cha and Chulhee Yun},
booktitle={ICML},
year={2025}
}

@inproceedings{nhuyen2019tight,
 author = {Nguyen, PHUONG\_HA and Nguyen, Lam and van Dijk, Marten},
 booktitle = {NeurIPS},
 title = {Tight Dimension Independent Lower Bound on the Expected Convergence Rate for Diminishing Step Sizes in SGD},
 year = {2019}
}

@article{Jentzen2020lower,
title = {Lower error bounds for the stochastic gradient descent optimization algorithm: Sharp convergence rates for slowly and fast decaying learning rates},
journal = {Journal of Complexity},
year = {2020},
author = {Arnulf Jentzen and Philippe {von Wurstemberger}}
}

@inproceedings{
koloskova2024on,
title={On Convergence of Incremental Gradient for Non-convex Smooth Functions},
author={Anastasia Koloskova and Nikita Doikov and Sebastian U Stich and Martin Jaggi},
booktitle={ICMLR},
year={2024}
}

@inproceedings{liu2024onthelastiterate,
author = {Liu, Zijian and Zhou, Zhengyuan},
title = {On the last-iterate convergence of shuffling gradient methods},
year = {2024},
booktitle = {ICML}
}

@InProceedings{safran2020howgood,
  title = 	 {How Good is SGD with Random Shuffling?},
  author =       {Safran, Itay and Shamir, Ohad},
  booktitle = 	 {Proceedings of Thirty Third Conference on Learning Theory},
  year = 	 {2020},
}

\newpage
\appendices
\onecolumn

\section{Experimental Setting}
\label{app:exp}

\paragraph{Hyperparameters.}
As per the hyperparameters, for \fedavg{} and \lenet{} we search the server step-size $\eta \in \{2, 1.5, 1, 0.5, 0.1\}$ and local step-size $\eta_l \in \{0.1, 0.05, 0.01, 0.005\}$ and found the best performing to be $\eta = 1.5$ and $\eta_l=0.01$. For \resnet{}, we search the server step-size $\eta \in \{1.5, 1, 0.1\}$ and local step-size $\eta_l \in \{1, 0.5, 0.1, 0.01\}$ and found the best performing to be $\eta = 1$ and $\eta_l=0.5$.
Similarly, for \fedcm{} and \lenet{} we search the server step-size $\eta \in \{1, 0.5, 0.1, 0.05\}$ and local step-size $\eta_l \in \{1, 0.5, 0.1, 0.05\}$ and found the best performing to be $\eta = 0.1$ and $\eta_l=0.1$. For \resnet{}, we search the server step-size $\eta \in \{1.5, 1, 0.5, 0.1\}$ and local step-size $\eta_l \in \{1, 0.5, 0.1, 0.5\}$ and found the best performing to be $\eta = 1$ and $\eta_l=0.1$. The momentum factor is searched among $\beta \in \{0.95, 0.9, 0.85\}$ and set as $\beta=0.9$.

\paragraph{Metrics and Experimental protocol.}
We consider the model accuracy in predicting the correct class images belong to. Results are always reported as average of $5$ independent runs, with stardard deviation directly shown in \cref{fig:fl_exp}.

\newpage
\section{Deferred Proofs}
\label{appendix:theory}

\subsection{Auxiliary Lemmas}
\label{appendix:theory:aux_lemmas}

Here is a collection of some smaller technical lemmas that are used within the proofs of the main results.

\begin{lemma}
\label{lemma:sum_lt_int_2}
    Let $f(x)$ be a non-negative, monotonically decreasing function that is integrable over an interval $[a,b]$, where $a<b$ are integers. The following inequality holds:
    \begin{align*}
        &\sum_{k=a+1}^{b} f(k) \leq \int_{a}^{b} f(x)\, dx\,
    \end{align*}
\end{lemma}
\begin{proof}
Since $f(x)$ is a monotonically decreasing function on the interval $[a,b]$, for any integer $k \in [a+1,b]$, and for any $x \in [k-1,k]$, we have that:
    \begin{align}
        f(k) \leq f(x)  \Rightarrow  
        \int_{k-1}^{k} f(k)\, dx &\leq \int_{k-1}^{k}f(x)\, dx \hspace{3cm}
    \intertext{Since $f(k)$ is constant \wrt{} the integration variable $x$, we have that:}
        f(k) &\leq \int_{k-1}^{k} f(x)\, dx\\
    \intertext{Summing up from $k=a+1$ to $k=b$ and using the additive property of integrals:}
        \sum_{k=a+1}^{b}f(k) &\leq \sum_{k=a+1}^{b}\int_{k-1}^{k}  f(x)\, dx \\
        &=\int_{a}^{b} f(x)\, dx 
    \end{align}
\end{proof}

\begin{lemma}
\label{lemma:lim_trans_matrix_1}
Let $\eta<\frac{2^\alpha}{\mu}$, and let the function $\Psi_1(t,s, \alpha)$ be:
\begin{equation*}
    \Psi_1(t,s, \alpha):= \prod_{k=s+1}^{t}\left(1-\frac{\mu{\eta}}{k^\alpha}\right)
\end{equation*}

Then, for any $t>=2$ and $1 \leq s<t$ the following holds:
    \begin{align*}
         \Psi_1(t,s, \alpha) &\leq
        \left\{\begin{alignedat}{3}
        &\exp\left( -\mu\eta \frac{t^{1-\alpha} - s^{1-\alpha}}{1-\alpha}\right) &\quad& \text{if } 0 < \alpha < 1 \\
        &\left(\frac{s}{t}\right)^{\mu\eta} &\quad& \text{if }\alpha = 1 \\
        &\exp\left(-\frac{\mu\eta}{2^\alpha}\right) &\quad& \text{if }\alpha > 1\\
        \end{alignedat}\right. \\ \\
        \Psi_1(t,s, \alpha) &\geq
        \left\{\begin{alignedat}{3}
        &\exp\left( -\frac{2^\alpha\mu\eta}{2^\alpha-\mu\eta} \frac{t^{1-\alpha} - s^{1-\alpha}}{1-\alpha}\right) &\quad& \text{if } 0 < \alpha < 1 \\
        &\left(\frac{s}{t}\right)^{\frac{2\mu\eta}{2-\mu\eta}} &\quad& \text{if }\alpha = 1 \\
        &\exp\left(-\frac{2^\alpha\mu\eta}{2^\alpha-\mu\eta}\zeta_H(\alpha, s+1)\right) &\quad& \text{if }\alpha > 1\\
        \end{alignedat}\right.
    \end{align*}
    where $\zeta_H(\alpha,s):= \sum_{k=s}^\infty \frac{1}{k^\alpha}$ is the Hurwitz zeta-function.
\end{lemma}
\begin{proof}
    \underline{\textbf{Case $0< \alpha \leq 1:$}} \quad For the upper bound, we have that
    \begin{align}
        \Psi_1(t,s, \alpha) &= \prod_{k=s+1}^{t} \left( 1-\frac{\mu\eta}{k^\alpha}\right) \\
        &= \exp\left( \sum_{k=s+1}^{t} \ln \left( 1- \frac{\mu\eta}{k^\alpha}\right)\right) \label{eq:lemma:lim_trans_matrix_1:prod_to_sum} \\
        &\leq \exp\left( - \sum_{k=s+1}^{t} \frac{\mu\eta}{k^\alpha}\right) \label{eq:lemma:lim_trans_matrix_1:first_step_ub}\\
        & \overset{\ref{lemma:sum_lt_int_2}}{\leq} \exp\left( -\mu\eta \int_{s}^{t} \frac{1}{k^\alpha} \,dk\right)\\
        &=\left\{
        \begin{alignedat}{2}
            & \exp\left( -\mu\eta \frac{t^{1-\alpha} - s^{1-\alpha}}{1-\alpha}\right) &\quad& \text{if } 0 < \alpha < 1 \\
            & \exp\left( -\mu\eta \log\left(\frac{t}{s}\right)\right) =\left(\frac{s}{t}\right)^{\mu\eta}  &\quad& \text{if } \alpha = 1
        \end{alignedat}
        \right.
    \end{align}
    where in the step (\ref{eq:lemma:lim_trans_matrix_1:first_step_ub}) we used the inequality $\ln(1-x) \leq -x$ for $x >0$, with $x=\frac{\mu\eta}{k^\alpha}$. 
    Similarly, for the lower bound we have that
    \begin{align}
        \Psi_1(t,s, \alpha) &= \prod_{k=s+1}^{t} \left( 1-\frac{\mu\eta}{k^\alpha}\right) \\
        &= \exp\left( \sum_{k=s+1}^{t} \ln \left( 1- \frac{\mu\eta}{k^\alpha}\right)\right) \\
        & \geq \exp\left(\sum_{k=s+1}^{t} \underbrace{-\frac{1}{1 - \mu\eta/k^\alpha}}_{\text{increasing}}\frac{\mu\eta}{k^\alpha}\right) \label{eq:step_intermediate} \\
        &\overset{k\geq s+1\geq2}{\geq} \exp\left( - \frac{2^\alpha}{2^\alpha-\mu\eta} \sum_{k=s+1}^{t} \frac{\mu\eta}{k^\alpha}\right) \label{eq:lemma:lim_trans_matrix_1:first_step_lb}\\
        &\overset{\text{\ref{lemma:sum_lt_int_2}}}{\geq} \exp\left( -\frac{2^\alpha\mu\eta}{(2^\alpha-\mu\eta)} \int_{s}^{t} \frac{1}{k^\alpha} \,dk\right) \\
        &=\left\{
        \begin{alignedat}{2}
            & \exp\left( -\frac{2^\alpha\mu\eta}{2^\alpha-\mu\eta} \frac{t^{1-\alpha} - s^{1-\alpha}}{(1-\alpha)}\right) &\quad& \text{if } 0 < \alpha < 1 \\
            & \exp\left( -\frac{2\mu\eta}{2-\mu\eta} \ln\left(\frac{t}{s}\right)\right) = \left(\frac{s}{t}\right)^{\frac{2\mu\eta}{2-\mu\eta}}  &\quad& \text{if } \alpha = 1
        \end{alignedat}
        \right.
    \end{align}
    where in the step (\ref{eq:step_intermediate}) we used the inequality $\log(1-x) \geq -\frac{x}{1-x}$ for $x>0$, with $x=\frac{\mu\eta}{k^\alpha}$. 

    \underline{\textbf{Case $\alpha > 1:$}} \quad
    For $\alpha > 1$, we have that 
    \begin{equation}
        \label{eq:series_boundaries}
        (s+1)^{-\alpha}\leq\sum_{k=s+1}^{t}k^{-\alpha} <\sum_{k=s+1}^{\infty}k^{-\alpha} < \infty \,\,\forall\,t
    \end{equation}
    Therefore, for the upper bound we have
    \begin{equation}
        \Psi_1(t,s,\alpha) \overset{\text{\cref{eq:lemma:lim_trans_matrix_1:first_step_ub}}}{\leq} \exp\left( - \sum_{k=s+1}^{t} \frac{\mu\eta}{k^\alpha}\right)
        \overset{\text{\cref{eq:series_boundaries}}}{\leq} \exp\left(-\frac{\mu\eta}{(s+1)^\alpha} \right) \overset{s\geq1}{\leq} \exp\left(-\frac{\mu\eta}{2^\alpha}\right)
    \end{equation}

    For the lower bound we have
    \begin{align}
        \Psi_1(t,s,\alpha) &\overset{\text{\cref{eq:lemma:lim_trans_matrix_1:first_step_lb}}}{\geq} \exp\left( - \frac{2^\alpha\mu\eta}{2^\alpha-\mu\eta} \sum_{k=s+1}^{t} \frac{1}{k^\alpha}\right) \\
        & \overset{\text{\cref{eq:series_boundaries}}}{\geq} \exp\left( - \frac{2^\alpha\mu\eta}{2^\alpha-\mu\eta} \sum_{k=s+1}^{\infty} \frac{1}{k^\alpha}\right) \\
        &= \exp\left(- \frac{2^\alpha\mu\eta}{2^\alpha-\mu\eta}\zeta_H(\alpha,s+1) \right)
    \end{align}
\end{proof}

\begin{corollary}
\label{corollary:lim_trans_matrix_1}
    Let $\eta<\frac{2^\alpha}{\mu}$, and let the function $\Psi_1(t,s, \alpha)$ be:
    \begin{equation*}
        \Psi_1(t,s, \alpha):= \prod_{k=s+1}^{t}\left(1-\frac{\mu{\eta}}{k^\alpha}\right)
    \end{equation*}
    Then the following holds:
    \begin{equation*}
        \lim_{t \to \infty} \Psi_1(t,s, \alpha)= \lim_{t \to \infty} \left\{\begin{alignedat}{2}
        &\exp\left( -t^{1-\alpha}\right) &\quad& \text{if } 0 < \alpha < 1 \\
        &\left(\frac{1}{t}\right)^{\mu\eta} &\quad& \text{if }\alpha = 1
        \end{alignedat}\right.
    \end{equation*}
    Moreover, for $\alpha>1$, it holds that:
    \begin{align*}
        \lim_{t\to \infty}\Psi_1(t,s,\alpha)&=c, \quad &c \in \left(\exp\left(-\frac{2^\alpha\mu\eta}{2^\alpha-\mu\eta}\zeta_H(\alpha, s+1)\right), \exp\left(-\frac{\mu\eta}{2^\alpha}\right)\right)
    \end{align*}
\end{corollary}
\begin{proof}
    The proof of the statement follows from taking the limit for $t \to \infty$, for each range of $\alpha$, of upper bounds in \cref{lemma:lim_trans_matrix_1}, which give the slowest decay.
\end{proof}

\begin{lemma}
\label{lemma:lim_trans_matrix_2}
Let $\mu,\eta$ and $\beta$ positive constants, and let the function $\Psi_2(t,s, \alpha)$ be:
\begin{equation*}
    \Psi_2(t,s, \alpha):= \prod_{k=s+1}^{t}\left( \beta + \frac{\mu\eta\beta}{k^\alpha}\right)
\end{equation*}
For any $1 \leq s < t$ and $\alpha > 0$, the following holds:
    \begin{align*}
         \Psi_2(t,s, \alpha) &\leq 
        \left\{\begin{alignedat}{3}
            &\beta^{t-s}\exp\left(\mu\eta \frac{t^{1-\alpha}-s^{1-\alpha}}{1-\alpha}\right) &\quad& \text{if } 0 < \alpha < 1 \\
            &\beta^{t-s}\left(\frac{t}{s}\right)^{\mu\eta} &\quad& \text{if } \alpha = 1 \\
            &\beta^{t-s}\exp\left(\mu\eta\zeta_H(\alpha,s+1)\right) &\quad& \text{if } \alpha > 1
        \end{alignedat}\right.\\
        \Psi_2(t,s, \alpha) &>0
    \end{align*}
     where $\zeta_H(\alpha,s):= \sum_{k=s}^\infty \frac{1}{k^\alpha}$ is the Hurwitz zeta-function.
\end{lemma}
\begin{proof}
    The fact that $\Psi_2(t,s, \alpha)$ is positive (lower bound) is trivial. For the upper bound, we can write
    \begin{align}
        \Psi_2(t,s, \alpha) &= \prod_{k=s+1}^{t}\left( \beta + \frac{\mu\eta\beta}{k^\alpha}\right) \\
        &= \beta^{t-s} \prod_{k=s+1}^{t}\left( 1 + \frac{\mu\eta}{k^\alpha}\right) \\
        &= \beta^{t-s} \exp\left(\sum_{k=s+1}^{t}\ln\left( 1 + \frac{\mu\eta}{k^\alpha}\right)\right) \\
        &\leq \beta^{t-s} \exp\left(\sum_{k=s+1}^{t} \frac{\mu\eta}{k^\alpha}\right)
        \label{eq:b2_intermediate}
    \end{align}
    where in the last step we used the inequality $\ln(x)\leq x-1\,\,\forall x>0$, with $x=\left(1+\frac{\mu\eta}{k^\alpha}\right)>0$, which is always verified since $\mu,\eta,k>0$.
    Now, we differentiate the next steps depending on the value of $\alpha$.

    \underline{\textbf{Case $0< \alpha < 1:$}} \quad Since the function within the summation in \cref{eq:b2_intermediate} is decreasing, we use \cref{lemma:sum_lt_int_2}:
    \begin{align}
        0<\Psi_2(t,s, \alpha) 
        &\leq \beta^{t-s} \exp\left(\mu\eta\sum_{k=s+1}^{t} \frac{1}{k^\alpha}\right)\\
        &\overset{\text{\ref{lemma:sum_lt_int_2}}}{\leq} §\beta^{t-s} \exp\left( \mu\eta\int_{s}^{t} \frac{1}{k^\alpha}\,dk\right) \\
        &= \beta^{t-s}\exp\left( \mu\eta\frac{t^{1-\alpha} - s^{1-\alpha}}{1-\alpha}\right)
    \end{align}
    
    \underline{\textbf{Case $\alpha=1:$}} \quad Using \cref{lemma:sum_lt_int_2} as in the previous case, we have that:
    \begin{align}
        0<\Psi_2(t,s, \alpha)  &\leq \beta^{t-s} \exp\left(\mu\eta\sum_{k=s+1}^{t} \frac{1}{k^\alpha}\right)\\
        &\overset{\text{\ref{lemma:sum_lt_int_2}}}{\leq} \beta^{t-s} \exp\left( \mu\eta\int_{s}^{t} \frac{1}{k}\, dk\right) \\
        &= \beta^{t-s}\exp\left( \mu\eta \ln\left(\frac{t}{s}\right)\right) \\
        &= \beta^{t-s} \left( \frac{t}{s}\right)^{\mu\eta}
    \end{align}

    \underline{\textbf{Case $\alpha > 1:$}} 
    \begin{align}
        0<\Psi_2(t,s, \alpha) &\leq \beta^{t-s} \exp\left(\mu\eta\sum_{k=s+1}^{t} \underbrace{\frac{1}{k^\alpha}}_{>0}\right) \\
        &\leq \beta^{t-s} \exp\left(\mu\eta\sum_{k=s+1}^{\infty} \frac{1}{k^\alpha}\right) \\
        &= \beta^{t-s} \exp\left(\mu\eta\zeta_H(\alpha,s+1)\right)
    \end{align}
This concludes the proof.
\end{proof}

\begin{corollary}
\label{corollary:lim_trans_matrix_2}
    Let $\mu,\eta$ and $\beta$ positive constants, and let the function $\Psi_2(t,s, \alpha)$ be:
    \begin{equation*}
        \Psi_2(t,s, \alpha):= \prod_{k=s+1}^{t}\left( \beta + \frac{\mu\eta\beta}{k^\alpha}\right)
    \end{equation*}
    Then the following holds:
        \begin{equation*}
            \lim_{t \to \infty} \Psi_2(t,s, \alpha) =
            \lim_{t \to \infty} \beta^t
        \end{equation*}
\end{corollary}
\begin{proof}
    The proof of the statement follows from taking the limit for $t \to \infty$, for each range of $\alpha$, of both upper and lower bounds of $\Psi_2(t,s,\alpha)$ in \cref{lemma:lim_trans_matrix_2} and using the squeeze theorem.
\end{proof}

\begin{lemma}
Let $\Psi_1(t,s,\alpha)$ as defined in \cref{lemma:lim_trans_matrix_1}, and let the summation $S(t,\alpha)$ be:
\begin{equation*}
    S(t,\alpha):=\sum_{s=2}^{t} \Psi_1(t,s,\alpha)\frac{1}{s^{2\alpha}}
\end{equation*}
\label{lemma:lim_sum_trans_matrix_1}
    Then, for any $\alpha>1$ the following holds:
    \begin{align*}
        \frac{\exp\left( -\frac{2^\alpha\mu\eta}{2^\alpha-\mu\eta}\zeta_R\left(\alpha\right)\right)}{2\alpha - 1} \leq \lim_{t \to \infty} {S}(t,\alpha) \leq \zeta_R(2\alpha)
    \end{align*}
\end{lemma}
\begin{proof}
    For $\alpha >1$, we have that:
    \begin{align}
         S(t,\alpha) 
         &\overset{\text{\ref{lemma:lim_trans_matrix_1}}}{\leq} \sum_{s=2}^{t} \frac{1}{s^{2\alpha}}\exp\left(-\frac{\mu\eta}{2^\alpha} \right) \\
         &\leq \sum_{s=1}^{\infty} \frac{1}{s^{2\alpha}} = \zeta_R(2\alpha) \label{eq:lemma:lim_sum_trans_matrix_1:final_ub_alphagt1}
    \end{align}
    On the other hand:
    \begin{align}
        S(t,\alpha) &\overset{\text{\ref{lemma:lim_trans_matrix_1}}}{\geq} \sum_{s=2}^{t}\frac{1}{s^{2\alpha}} \exp\left( -\frac{2^\alpha\mu\eta}{2^\alpha-\mu\eta}\zeta_H\left(\alpha,s \right)\right)\\
        &\geq \exp\left( -\frac{2^\alpha\mu\eta}{2^\alpha-\mu\eta}\zeta_R\left(\alpha\right)\right)\sum_{s=2}^{t}\frac{1}{s^{2\alpha}} \\
        &\geq \exp\left( -\frac{2^\alpha\mu\eta}{2^\alpha-\mu\eta}\zeta_R\left(\alpha\right)\right)\int_{2}^{t}\frac{1}{s^{2\alpha}}\,ds \\
        &= \exp\left( -\frac{2^\alpha\mu\eta}{2^\alpha-\mu\eta}\zeta_R\left(\alpha\right)\right) \frac{t^{1-2\alpha}-2^{1-\alpha}}{1-2\alpha} \\
        &= \exp\left( -\frac{2^\alpha\mu\eta}{2^\alpha-\mu\eta}\zeta_R\left(\alpha\right)\right) \left( \frac{1}{2\alpha -1} - \frac{2^{1-\alpha}}{(2\alpha-1)t^{2\alpha-1}}\right) \label{eq:lemma:lim_sum_trans_matrix_1:final_lb_alphagt1}
    \end{align}
    Since $\lim_{t \to \infty}1/t^{2\alpha-1}=0$ because $2\alpha-1>0$ since $\alpha>1$. Putting together the results of \cref{eq:lemma:lim_sum_trans_matrix_1:final_lb_alphagt1,eq:lemma:lim_sum_trans_matrix_1:final_ub_alphagt1}, we have that:
    \begin{align}
        \frac{\exp\left( -\frac{2^\alpha\mu\eta}{2^\alpha-\mu\eta}\zeta_R\left(\alpha\right)\right)}{2\alpha - 1} \leq \lim_{t \to \infty} S(t,\alpha) \leq \zeta_R(2\alpha)
    \end{align}
\end{proof}

\begin{lemma}
\label{lemma:lim_sum_trans_matrix_2}
Let $\alpha>0, n>0$ and $\Psi_2(t,s,\alpha) := \prod_{k=s+1}^{t}\left(\beta+\frac{\mu{\eta}\beta}{k^\alpha}\right)$. Then with $\beta \in [0,1)$ the following holds:
    \begin{equation*}
        \lim_{t \to \infty} \sum_{s=2}^{t}\Psi_2(t,s,\alpha) \frac{1}{s^n} = \lim_{t \to \infty} \frac{1}{t^n}
    \end{equation*}
\end{lemma}
\begin{proof}
    For $\beta=0$ the statement is trivially true, because $\Psi_2(t,s,\alpha)=0$. Therefore, from this point on, we consider $\beta \in (0,1)$.
    For readability, let us define the shorthand notation for the quantity in the l.h.s. of the statement
    \begin{equation}
        S(t,\alpha,n):= \sum_{s=2}^{t}\Psi_2(t,s,\alpha) \frac{1}{s^n}
    \end{equation}
    For any $\alpha>0$ we have that:
    \begin{align}
        \bar{S}(t,\alpha,n) &:=t^{n}S(t,\alpha) = \sum_{s=2}^{t}
        \underbrace{\prod_{k=s+1}^{t}\left( \beta + \frac{\mu\eta\beta}{k^\alpha}\right) \left(\frac{t}{s}\right)^n}_{>0} \label{eq:sbar_1}
    \end{align}
    
    To prove the statement we derive the convergence of $\bar{S}(t,\alpha,n)$, and consequently $S(t,\alpha,n)$, by using the \textbf{Dominated Convergence Theorem}. We recall that this theorem states that given a sequence $f_t(u)$ such that 
    \begin{description}
        \item[\quad cond. 1) ] $\lim_{t \to \infty} f_t(u) = f(u) < \infty$
        \item[\quad cond. 2) ] there exist a summable function  $g(u) \geq |f_t(u)|$
    \end{description}
    then it holds that: 
    \begin{equation}
        \lim_{t \to \infty} \sum_{u=0}^{\infty} f_t(u) = \sum_{u=0}^{\infty} \lim_{t \to \infty} f_t(u) = \sum_{u=0}^{\infty} f(u) \label{eq:lemma:lim_sum_alt_trans_matrix_2_dom_conv_thm}
    \end{equation}
    We proceed to breakdown the analysis for different ranges of the variable $\alpha$.
        
    \vspace{5px}\underline{\textbf{Case $0<\alpha<1:$}}\quad
    Starting from \cref{eq:sbar_1}, we have that:
    \begin{align}
        \bar{S}(t,\alpha) &=  \sum_{s=2}^{t}\prod_{k=s+1}^{t}\left( \beta + \frac{\mu\eta\beta}{k^\alpha}\right) \left(\frac{t}{s}\right)^n \\
        &\overset{\text{\ref{lemma:lim_trans_matrix_2}}}{\leq} \sum_{s=2}^{t} \beta^{t-s} \exp\left( \mu\eta \frac{t^{1-\alpha} - s^{1-\alpha}}{1-\alpha}\right) \left(\frac{t}{s}\right)^n \\
        &\overset{u:=t-s}{=} \;\;  \sum_{u=0}^{t-2} \beta^{u} \exp\left( \mu\eta \frac{t^{1-\alpha} - (t-u)^{1-\alpha}}{1-\alpha}\right) \left(\frac{t}{t-u}\right)^{n} \\
        &\overset{c:= \frac{\mu\eta}{1-\alpha}}{=} \sum_{u=0}^{t-2} \underbrace{\beta^{u} \exp\Big( c(t^{1-\alpha}-(t-u)^{1-\alpha})\Big) \left(\frac{t}{t-u}\right)^{n}}_{:=f_t(u)}
    \end{align}
    For the \emph{condition 1}, we have that
    \begin{align*}
        f(u) &:= \lim_{t\to\infty} \beta^{u} \exp\Big( c(t^{1-\alpha}-(t-u)^{1-\alpha})\Big) \left(\frac{t}{t-u}\right)^{n} = \beta^u
    \intertext{For the \emph{condition 2}, we observe that:} 
        f_t(u) &= \underbrace{\beta^{u}}_{>0} \underbrace{\exp\left( c(t^{1-\alpha}-(t-u)^{1-\alpha})\right)}_{\geq 1} \underbrace{\left(\frac{t}{t-u}\right)^{n}}_{\geq1}
    \end{align*}
    and that the last two terms have a maximum in $u=t-2$. Thus, it follows that
    \begin{equation}
        \begin{aligned}
            \big|f_t(u)\big| &\leq \beta^u \exp\Big( c (t^{1-\alpha}-2^{1-\alpha})\Big)\left(\frac{t}{2}\right)^n \\
            &\leq \beta^u  \exp\Big( c t^{1-\alpha}\Big)t^n := g(u)
        \end{aligned}
    \end{equation}
    To verify that $g(u)$ is summable, we can apply the ratio test:
    \begin{equation*}
        \lim_{u \to \infty} \left| \frac{g(u+1)}{g(u)} \right| = \lim_{u \to \infty}
        \frac{\beta^{u+1}  \exp\Big( c t^{1-\alpha}\Big)t^n}
        {\beta^u  \exp\Big( c t^{1-\alpha}\Big)t^n} =\beta
    \end{equation*}
    Since the ratio is $\beta <1$, this confirms that $g(u)$ is summable. 
    Therefore we obtain that:
    \begin{align*}
         \lim_{t \to \infty} S(t, \alpha,n) &= \lim_{t \to \infty}t^{-n}\sum_{u=0}^{\infty} f(u) \\
         &= \lim_{t \to \infty}\frac{1}{(1-\beta)t^n} = \lim_{t \to \infty}\frac{1}{t^n}
    \end{align*}

\underline{\textbf{Case $\alpha=1:$}}\quad For $\alpha=1$ we proceed similarly to the previous case. From \cref{eq:sbar_1} we have that:
\begin{align}
    \bar{S}(t,1,n) &=  \sum_{s=2}^{t}\prod_{k=s+1}^{t}\left( \beta + \frac{\mu\eta\beta}{k^\alpha}\right) \left(\frac{t}{s}\right)^n \\
    &\overset{\text{\ref{lemma:lim_trans_matrix_2}}}{\leq} \sum_{s=1}^{t} \beta^{t-s} \left( \frac{t}{s}\right)^{\mu\eta+n} \\
    &\overset{u:=t-s}{=} \sum_{u=0}^{t-2} \underbrace{\beta^{u} \left(\frac{t}{t-u}\right)^{\mu\eta+n}}_{:=f_t(u)\geq1}
\end{align}

For the \emph{condition 1}, we have that:
\begin{align*}
    f(u) &:= \lim_{t\to\infty} f_t(u) = \beta^u
\intertext{For the \emph{condition 2}, the second term in $f_t(u)$ has a maximum in $u=t-2$, i.e.,}
    \big|f_t(u)\big| &\leq \beta^u \left(\frac{t}{2}\right)^{\mu\eta+n} 
    \leq \beta^u t^{\mu\eta+n}:= g(u)
\end{align*}
Hence,  going back to $S(t,\alpha)$ with \cref{eq:sbar_1,eq:lemma:lim_sum_alt_trans_matrix_2_dom_conv_thm}, we have that:
\begin{align*}
    \lim_{t \to \infty} S(t, 1,n) &= \lim_{t \to \infty} t^{-n}\bar{S}(t,1,n)  \\ &= 
    \lim_{t \to \infty} t^{-n} \sum_{u=0}^{\infty} f(u) \\
    &= \lim_{t \to \infty} t^{-n} \sum_{u=0}^{\infty} \beta^u \\
    &= \lim_{t \to \infty} \frac{1}{(1-\beta)t^n} = \lim_{t \to \infty} \frac{1}{t^n}
\end{align*}

\underline{\textbf{Case $\alpha>1:$}}\quad The case $\alpha>1$ is analogous, and differs from the above only for a constant factor. We have that:
    \begin{align*}
       \bar{S}(t,\alpha,n) &=  \sum_{s=2}^{t}\prod_{k=s+1}^{t}\left( \beta + \frac{\mu\eta\beta}{k^\alpha}\right) \left(\frac{t}{s}\right)^n \\
       &\overset{\text{\ref{lemma:lim_trans_matrix_2}}}{\leq} \sum_{s=2}^{t} \beta^{t-s} \exp\Big(\mu\eta \zeta_H(\alpha,s+1)\Big) \left(\frac{t}{s}\right)^n \\
       &\overset{u:=t-s}{=} \sum_{u=0}^{t-2} \underbrace{\beta^u \exp\Big(\mu\eta \zeta_H(\alpha,t-u+1)\Big) \left(\frac{t}{t-u}\right)^n}_{:=f_t(u)}
    \end{align*}
    For  the \emph{condition 1}, we have that $f(u) := \lim_{t\to\infty} f_t(u) =\beta^u$. For the \emph{condition 2}, the maximum of the second and third terms of $f_t(u)$ is found at $u=t-2$, hence we have
    \begin{equation*}
        \big|f_t(u)\big| \leq \beta^u \exp\Big(\mu\eta \zeta_H(\alpha,3)\Big) \left(\frac{t}{2}\right)^n 
    \leq \beta^u \exp\Big(\mu\eta \zeta_H(\alpha,1)\Big) t^n  := g(u)
\end{equation*}
    
    Finally, going back to $S(t,\alpha)$ with \cref{eq:sbar_1,eq:lemma:lim_sum_alt_trans_matrix_2_dom_conv_thm}, we have that:
    \begin{align*}
        \lim_{t \to \infty} S(t, \alpha,n) &= \lim_{t\to\infty} t^{-n} \bar{S}(t,\alpha,n)
        =\lim_{t \to \infty} t^{-n} \sum_{u=0}^{\infty} f(u) \\
        &= \lim_{t \to \infty} t^{-n} \sum_{u=0}^{\infty} \beta^u \\
        &= \lim_{t \to \infty} \frac{1}{(1-\beta)t^n} = \lim_{t \to \infty} \frac{1}{t^n}
    \end{align*}
\end{proof}

\begin{lemma}
\label{lemma:lim_sum_alt_trans_matrix_2}
Let $\alpha>0, n>0$ and $\Psi_2(t,s,\alpha) := \prod_{k=s+1}^{t}\left(\beta+\frac{\mu{\eta}\beta}{k^\alpha}\right)$. Then with $\beta \in [0,1)$ the following holds:
    \begin{equation*}
        \lim_{t \to \infty} \sum_{s=2}^{t}\Psi_2(t,s,\alpha) \frac{(-1)^s}{s^n} = \lim_{t \to \infty} \frac{(-1)^t}{t^n}
    \end{equation*}
\end{lemma}
\begin{proof}
    For readability, let us define the shorthand notation for the quantity in the l.h.s. of the statement
    Notice that
    \begin{align}
        S(t,\alpha,n) &:= \sum_{s=2}^{t}\Psi_2(t,s,\alpha) \frac{(-1)^s}{s^n} \\
        \intertext{and notice that:}
        \left| S(t,\alpha,n) \right| &\leq \sum_{s=2}^{t}\Psi_2(t,s,\alpha) \frac{1}{s^n}
    \end{align}
    Therefore, from direct application of \cref{lemma:lim_sum_trans_matrix_2} on the r.h.s. and the squeeze-theorem we have that for any $\alpha>0,n>0$:
    \begin{align}
        \lim_{t \to \infty} S(t,\alpha) &= \pm \lim_{t \to \infty} \frac{1}{t^n} \\
        &= \lim_{t \to \infty} \frac{(-1)^t}{t^n}
    \end{align}
\end{proof}

\begin{lemma}
\label{lemma:lim_sum_alt_trans_matrix_gen}
Let $\alpha>0$ and $\Psi_1(t,s,\alpha):=\prod_{k=s+1}^{t}\left(1-\frac{\mu{\eta}}{k^\alpha}\right)$. Consider the function
    \begin{equation*}
        S(t,\alpha,n):=\sum_{s=2}^{t} \frac{(-1)^s}{s^n} \Psi_1(t,s,\alpha)
    \end{equation*}
with $t\geq 2$ and $1\leq s < t$. Then, for $n>0$, the following holds:
    \begin{align*}
     &\lim_{t\to\infty}S(t,\alpha,n) =
        \lim_{t \to \infty} \frac{(-1)^t}{t^n} & \text{if } 0<\alpha < 1\\
    &\lim_{t\to\infty}S(t,1,n) =\lim_{t \to \infty}
    \left\{
        \begin{alignedat}{2}
            & \frac{(-1)^t}{t^n} &\quad& \text{if } n < \mu\eta \\
            & \frac{1}{t^{\mu\eta}}  &\quad& \text{otherwise }
        \end{alignedat}
    \right. & \text{if } \alpha=1\\
    \gamma_1\exp\left(-\frac{2^\alpha\mu\eta}{2^\alpha-\mu\eta}\zeta_H(\alpha,3)\right) 
    \leq &\lim_{t\to\infty}S(t,\alpha,n)
    \leq  \frac{1}{2^n}\exp\left(-\frac{\mu\eta}{2^\alpha}\right)  & \text{if } \alpha>1
    \end{align*}
    where $\gamma_1:=\left( 
    \left(\frac{1}{2}\right)^n 
    -
    \left(\frac{1}{3}\right)^n  \frac{3^\alpha}{3^\alpha - \mu\eta}
    \right) >0$ and $\zeta_H(\alpha,s):= \sum_{k=s}^\infty \frac{1}{k^\alpha}$ is the Hurwitz zeta-function.
\end{lemma}
\begin{proof}
    From the definition, rewrite $S(t,\alpha,n)$ as recurrence:
    \begin{align}
        S(t,\alpha,n) &= \sum_{s=2}^{t} \prod_{k=s+1}^{t}\left(1-\frac{\mu\eta}{k^\alpha}\right)\frac{(-1)^s}{s^n} \\
        &=\sum_{s=2}^{t-1} \prod_{k=s+1}^{t}\left(1-\frac{\mu\eta}{k^\alpha}\right)\frac{(-1)^s}{s^n} + \left[\frac{(-1)^s}{s^n}\prod_{k=s+1}^{t}\left(1-\frac{\mu\eta}{k^\alpha}\right)\right]_{s=t} \\
        &=\sum_{s=2}^{t-1} \prod_{k=s+1}^{t}\left(1-\frac{\mu\eta}{k^\alpha}\right)\frac{(-1)^s}{s^n} +\frac{(-1)^t}{t^n}\underbrace{\prod_{k=t+1}^{t}\left(1-\frac{\mu\eta}{k^\alpha}\right)}_{=1} \\
        &=\left(1-\frac{\mu\eta}{t^\alpha}\right)\sum_{s=2}^{t-1}  \prod_{k=s+1}^{t-1}\left(1-\frac{\mu\eta}{k^\alpha}\right)\frac{(-1)^s}{s^n} + \frac{(-1)^t}{t^n} \\
        &= \left(1-\frac{\mu\eta}{t^\alpha}\right) S(t-1,\alpha,n) + \frac{(-1)^t}{t^n} \label{eq:lemma:lim_sum_alt_trans_matrix_gen:recurrence}\\
        &\overset{\text{unrolling}}{=}\prod_{k=3}^t\left(1-\frac{\mu\eta}{k^\alpha}\right)S(2,\alpha,n) + 
        \sum_{k=3}^t \prod_{s=k+1}^{t}\left(1-\frac{\mu\eta}{s^\alpha}\right)\frac{(-1)^k}{k^n}
        \label{eq:S_psi1_unrolled}
    \end{align}
    The solution of the above first-order non-homogeneous recurrence is the sum of the homogeneous solution $S^{(h)}(t,\alpha,n)$ and a particular solution $S^{(p)}(t,\alpha,n)$, which can be analyzed separately.
    From \cref{eq:S_psi1_unrolled}, we have that:
    \begin{align}
        S^{(h)}(t,\alpha,n) 
        &= \prod_{k=3}^t \left(1-\frac{\mu\eta}{k^\alpha}\right) S(2,\alpha,n) \\
        &= \Psi_1(t,2,\alpha) \frac{1}{2^n} \label{eq:lemma:lim_sum_alt_trans_matrix_gen:homo_sol}
    \end{align}
    For the particular solution, we look for a form $S^{(p)}(t,\alpha,n)=\gamma \frac{(-1)^t}{t^n}$, and by substituting into the original recurrence in \cref{eq:lemma:lim_sum_alt_trans_matrix_gen:recurrence} we have:
    \begin{align}
        \gamma \frac{(-1)^t}{t^n} &=  \gamma\left(1-\frac{\mu\eta}{t^\alpha}\right)\frac{(-1)^{t-1}}{(t-1)^n} + \frac{(-1)^t}{t^n}
    \intertext{Dividing by $\frac{(-1)^t}{t^n}$:}
        \gamma &=  -\gamma\left(1-\frac{\mu\eta}{t^\alpha}\right)\left(\frac{t}{t-1}\right)^n + 1
    \end{align}
So, for $t \to \infty$, $\gamma \to \frac{1}{2}$ and:
\begin{equation}
    \lim_{t \to \infty} S^{(p)}(t,\alpha,n) = \lim_{t \to \infty}\frac{1}{2}\frac{(-1)^t} {t^n} \label{eq:lemma:lim_sum_alt_trans_matrix_gen:part_sol}
\end{equation}
The asymptotic behavior of the homogeneous solution, and so of the original recurrence, depends on $\alpha$.

\underline{\textbf{Case $0< \alpha <1:$}}\quad 
From \cref{eq:lemma:lim_sum_alt_trans_matrix_gen:homo_sol,eq:lemma:lim_sum_alt_trans_matrix_gen:part_sol}, we have that:
\begin{align}
    \lim_{t \to \infty} S(t,\alpha,n) &\overset{\text{\ref{corollary:lim_trans_matrix_1}}}{=} \lim_{t \to \infty} \left[ 
    \left(\frac{1}{2}\right)^n \exp\left( - t^{1-\alpha}\right) 
    +
    \frac{1}{2}\frac{(-1)^t}{t^n}
    \right] \\
    &=\lim_{t \to \infty} \frac{(-1)^t}{t^n}
\end{align}

\underline{\textbf{Case $ \alpha =1:$}}\quad 
From \cref{eq:lemma:lim_sum_alt_trans_matrix_gen:homo_sol,eq:lemma:lim_sum_alt_trans_matrix_gen:part_sol}, we have that:
\begin{align}
    \lim_{t \to \infty} S(t,\alpha,n) &\overset{\text{\ref{corollary:lim_trans_matrix_1}}}{=} \lim_{t \to \infty} \left[ 
    \left(\frac{1}{2}\right)^n \left(\frac{1}{t}\right)^{\mu\eta}
    +
    \frac{1}{2}\frac{(-1)^t}{t^n}
    \right] \\
    &=\lim_{t \to \infty}
    \left\{
        \begin{alignedat}{2}
            & \frac{(-1)^t}{t^n} &\quad& \text{if } n < \mu\eta \\
            & \frac{1}{t^{\mu\eta}}  &\quad& \text{otherwise }
        \end{alignedat}
    \right.
\end{align}

\underline{\textbf{Case $ \alpha >1:$}}\quad In this case the summation converges to a non-zero constant, and we use a different strategy.     Starting from the original definition of $\Psi_1(t,s,\alpha)$ as per \cref{lemma:lim_trans_matrix_1}, we call $g(t,s):=\left(\frac{1}{s}\right)^n \prod_{k=s+1}^{t}\left(1-\frac{\mu\eta}{k^\alpha}\right)$. Noticing that the function is decreasing in $s$ we have that:
\begin{align}
    S(t,\alpha,n)&=\sum_{s=2}^{t}\frac{(-1)^s}{s^n}\prod_{k=s+1}^{t}\left(1-\frac{\mu\eta}{k^\alpha}\right) \label{eq:lemma:lim_sum_alt_trans_matrix_gen:alphagt1_base} \\
    &= g(t,2)+\sum_{s=3}^t (-1)^s g(t,s) \\
    &\leq g(t,2)+\sum_{s=2}^{\lfloor t/2 \rfloor} \underbrace{\left(g(t,2s) - g(t,2s-1) \right)}_{<0} \\
    & \leq g(t,2) \overset{\text{\ref{lemma:lim_trans_matrix_1}}}{\leq} \frac{1}{2^n}\exp\left(-\frac{\mu\eta}{2^\alpha}\right)
\end{align}
Similarly, from \cref{eq:lemma:lim_sum_alt_trans_matrix_gen:alphagt1_base}, we have that:
\begin{align}
    S(t,\alpha, n) &= \sum_{s=2}^{t}\frac{(-1)^s}{s^n}\prod_{k=s+1}^{t}\left(1-\frac{\mu\eta}{k^\alpha}\right)  \\
    &\geq \sum_{s=1}^{\lfloor t/2 \rfloor} \underbrace{\left(g(t,2s) - g(2s+1)\right)}_{>0} \\
    & \geq g(t,2) - g(t,3)
\end{align}
So, defining $\gamma_1:=\left( 
    \left(\frac{1}{2}\right)^n 
    -
    \left(\frac{1}{3}\right)^n  \frac{3^\alpha}{3^\alpha - \mu\eta}
    \right)>0$, we have that:
\begin{align}
    g(t,2) - g(t,3) &= \left(\frac{1}{2}\right)^n  \prod_{k=3}^{t}\left(1-\frac{\mu\eta}{k^\alpha}\right)
    -
    \left(\frac{1}{3}\right)^n  \prod_{k=4}^{t}\left(1-\frac{\mu\eta}{k^\alpha}\right) \\
    &= \prod_{k=3}^{t}\left(1-\frac{\mu\eta}{k^\alpha}\right)
    \left( 
    \left(\frac{1}{2}\right)^n 
    -
    \left(\frac{1}{3}\right)^n  \left(1-\frac{\mu\eta}{3^\alpha}\right)^{-1}\right)\\
    &= \prod_{k=3}^{t}\left(1-\frac{\mu\eta}{k^\alpha}\right)
    \left( 
    \left(\frac{1}{2}\right)^n 
    -
    \left(\frac{1}{3}\right)^n  \frac{3^\alpha}{3^\alpha - \mu\eta}
    \right) \\
    & \overset{\text{\ref{lemma:lim_trans_matrix_1}}}{\geq} \gamma_1 \exp\left(\frac{2^\alpha \mu\eta}{2^\alpha-\mu\eta} \zeta_H(\alpha,3)\right)
\end{align}
\end{proof}

\subsection{Proofs of Main Theorems}
\label{appendix:theory:main_thms}

\begin{lemma}[One round progress of \fedavgm{}]
\label{lemma:update_fedavgm}
 For any positive constants $G, \mu$, define $\mu$-strongly convex functions $f_1(\theta) := \frac{\mu}{2} \theta^2 + G\theta$ and $f_2(\theta):=\frac{\mu}{2} \theta^2 - G\theta$ satisfying assumption \ref{assum:bounded_gd} and such that $f(\theta)=\frac{1}{2}\left(f_1(\theta) + f_2(\theta)\right)$. Under cyclic participation (assumption \ref{assum:cyclic_part}) with $C=0.5$, for any $t \ge 1$ the update of \fedavgm{} after $J$ local steps, with global and local step-sizes $\eta_t,\eta_l$ and momentum weight $\beta$ is:
 \begin{equation*}
     \theta^t = \theta^{t-1}\left(1+\beta-\tilde{\eta}_t+\tilde{\eta}_tD(J,J)\right) + (-1)^{t}\tilde{\eta}_t \eta_l G D(0,J-1) - \beta \theta^{t-2}
 \end{equation*}
 where $D(s,J):=\sum_{j=s}^{J}\left(1-\mu\eta_l\right)^j$ and $\tilde{\eta}_t=\eta_t(1-\beta)$.
\end{lemma}
\begin{proof}
    Start from the local and global model update rules of \fedavgm{}:
    \begin{align}
        \theta^{t,j}_i &= \theta^{t,j-1}_i - \eta_l \nabla f_i(\theta^{t,j-1}_i), \qquad \theta^{t,0}_i=\theta^{t-1} \label{proof:lemma_fedavgm_rule:client_update}\\
        \theta^t &=\theta^{t-1} - \frac{\tilde{\eta}_t}{|\S{t}|}\sum_{i \in \S{t}}\left(\theta^{t-1}-\theta^{t,J}_i\right) + \beta \left(\theta^{t-1}-\theta^{t-2}\right)\\
        &=\theta^{t-1}\left(1+\beta-\tilde{\eta}_t\right) + \frac{\tilde{\eta}_t}{|\S{t}|}\sum_{i \in \S{t}} \theta^{t,J}_i - \beta \theta^{t-2} \label{proof:lemma_fedavgm_rule:server_update}
    \end{align}
    Since in our setting there are only two clients $\S{}=\{1,2\}$ (or, equivalently, two sets of clients optimizing $f_1(\theta)$ or $f_2(\theta)$), from \cref{proof:lemma_fedavgm_rule:client_update} we have that:
    \begin{align}
        \theta^{t,J}_1 &= \theta^{t,J-1}_1 -\eta_l(\mu\theta^{t,J-1}_1+G)=\theta^{t,J-1}_1(1-\mu\eta_l)-\eta_l G\\
        &= \left(\theta^{t,J-2}_1(1-\mu\eta_l)-\eta_l G\right)\left(1-\mu\eta_l\right)-\eta_l G\\
        &= \theta^{t,J-2}_1\left(1-\mu\eta_l\right)^2 -\eta_l G \left(1+ (1-\mu\eta_l)\right) \\
        &\quad\vdots \nonumber \\
        &=\theta^{t-1}(1-\mu\eta_l)^J - \eta_l G \sum_{j=0}^{J-1}\left(1-\mu\eta_l\right)^j \label{proof:lemma_fedavgm_rule:client_update_f1}
    \end{align}
    where we used the fact that $\theta_1^{t,0}=\theta^{t-1}$. Similarly, for the client optimizing $f_2(\theta)$:
    \begin{align}
        \theta^{t,J}_2 
        &=\theta^{t-1}(1-\mu\eta_l)^J + \eta_l G \sum_{j=0}^{J-1}\left(1-\mu\eta_l\right)^j \label{proof:lemma_fedavgm_rule:client_update_f2}
    \end{align}
    Hence, sampling cyclically $f_1(\theta)$ at $t$ odd and $f_2(\theta)$ at $t$ even, plugging \cref{proof:lemma_fedavgm_rule:client_update_f1,proof:lemma_fedavgm_rule:client_update_f2} into \cref{proof:lemma_fedavgm_rule:server_update} we have that:
    \begin{align}
        \theta^t&=\theta^{t-1}\left(1+\beta-\tilde{\eta}_t\right) + \tilde{\eta}_t \left(\theta^{t-1}(1-\mu\eta_l)^J + (-1)^{t}\eta_l G \sum_{j=0}^{J-1}(1-\mu\eta_l)^j \right) - \beta \theta^{t-2} \\
        &=\theta^{t-1}\left(1+\beta-\tilde{\eta}_t + \tilde{\eta}_t (1-\mu\eta_l)^J\right) + (-1)^t \tilde{\eta}_t\eta_l G \sum_{j=0}^{J-1}(1-\mu\eta_l)^j - \beta \theta^{t-2}
    \end{align}
\end{proof}

\begin{lemma}[One round progress of \fedcm{}]
\label{lemma:update_fedcm}
 For any positive constants $G, \mu$, define $\mu$-strongly convex functions $f_1(\theta) := \frac{\mu}{2} \theta^2 + G\theta$ and $f_2(\theta):=\frac{\mu}{2} \theta^2 - G\theta$ satisfying assumption \ref{assum:bounded_gd} and such that $f(\theta)=\frac{1}{2}\left(f_1(\theta) + f_2(\theta)\right)$. Under cyclic participation (assumption \ref{assum:cyclic_part}) with $C=0.5$, for any $t \ge 1$ the update of \fedcm{} after $J$ local steps, with global and local step-sizes $\eta_t,\eta_l$ and momentum weight $\beta$ is:
 \begin{equation*}
     \theta^t = \theta^{t-1}\left(1+\hat{\beta}D(0,J-1)-\eta_t+\eta_tD(J,J)\right) + (-1)^{t}\eta_t \tilde{\eta}_l G D(0,J-1) - \hat{\beta} \theta^{t-2}D(0,J-1)
 \end{equation*}
 where $D(s,J):=\sum_{j=s}^{J}\left(1-\mu\tilde{\eta}_l\right)^j$, $\tilde{\eta}_l=\eta_l(1-\beta)$ and $\hat{\beta}:=\frac{\beta}{J}$.
\end{lemma}
\begin{proof}
    Start from the local and global model update rules of \fedcm{}:
    \begin{align}
        \theta^{t,j}_i &= \theta^{t,j-1}_i - \tilde{\eta}_l \nabla f_i(\theta^{t,j-1}_i) + \frac{\hat{\beta}}{\eta_t}\left(\theta^{t-1}-\theta^{t-2}\right), \qquad \theta^{t,0}_i=\theta^{t-1} \label{proof:lemma_fedcm_rule:client_update}\\
        \theta^t &=\theta^{t-1} - \frac{\eta_t}{|\S{t}|}\sum_{i \in \S{t}}\left(\theta^{t-1}-\theta^{t,J}_i\right) \\
        &= \theta^{t-1}\left(1-\eta_t \right) +\frac{\eta_t}{|\S{t}|}\sum_{i \in \S{t}}\theta^{t,J}_i \label{proof:lemma_fedcm_rule:server_update}
    \end{align}
    Since in our setting there are only two clients $\S{}=\{1,2\}$ (or, equivalently, two sets of clients optimizing $f_1(\theta)$ or $f_2(\theta)$), from \cref{proof:lemma_fedcm_rule:client_update} we have that:
    \begin{align}
        \theta^{t,J}_1 &= \theta^{t,J-1}_1 -\tilde{\eta}_l(\mu\theta^{t,J-1}_1+G) + \frac{\hat{\beta}}{\eta_t}\left(\theta^{t-1}-\theta^{t-2}\right)\\
        &=\theta^{t,J-1}_1(1-\mu\tilde{\eta}_l)-\tilde{\eta}_l G + \frac{\hat{\beta}}{\eta_t}\left(\theta^{t-1}-\theta^{t-2}\right)\\
        &= \left(\theta^{t,J-2}_1(1-\mu\tilde{\eta}_l)-\tilde{\eta}_l G+\frac{\hat{\beta}}{\eta_t}\left(\theta^{t-1}-\theta^{t-2}\right)\right)\left(1-\mu\tilde{\eta}_l\right)-\tilde{\eta}_l G + \frac{\hat{\beta}}{\eta_t}\left(\theta^{t-1}-\theta^{t-2}\right)\\
        &= \theta^{t,J-2}_1\left(1-\mu\tilde{\eta}_l\right)^2 +\left( \frac{\hat{\beta}}{\eta_t}\left(\theta^{t-1}-\theta^{t-2}\right)-\tilde{\eta}_l G\right) \left(1+ (1-\mu\tilde{\eta}_l)\right) \\
        &\quad\vdots \nonumber \\
        &=\theta^{t-1}\left((1-\mu\tilde{\eta}_l)^J + \frac{\hat{\beta}}{\eta_t}\sum_{j=0}^{J-1}\left( 1-\mu\tilde{\eta}_l\right)^j\right)- \tilde{\eta}_l G \sum_{j=0}^{J-1}\left(1-\mu\tilde{\eta}_l\right)^j - \frac{\hat{\beta}}{\eta_t}\theta^{t-2}\sum_{j=0}^{J-1}\left( 1-\mu\tilde{\eta}_l\right)^j\label{proof:lemma_fedcm_rule:client_update_f1}
    \end{align}
    where we used the fact that $\theta_1^{t,0}=\theta^{t-1}$. Similarly, for the client optimizing $f_2(\theta)$:
    \begin{align}
        \theta^{t,J}_2 
        &=\theta^{t-1}\left((1-\mu\tilde{\eta}_l)^J + \frac{\hat{\beta}}{\eta_t}\sum_{j=0}^{J-1}\left( 1-\mu\tilde{\eta}_l\right)^j\right)+ \tilde{\eta}_l G \sum_{j=0}^{J-1}\left(1-\mu\tilde{\eta}_l\right)^j - \frac{\hat{\beta}}{\eta_t}\theta^{t-2}\sum_{j=0}^{J-1}\left( 1-\mu\tilde{\eta}_l\right)^j \label{proof:lemma_fedcm_rule:client_update_f2}
    \end{align}
    Hence, sampling cyclically $f_1(\theta)$ at $t$ odd and $f_2(\theta)$ at $t$ even, plugging \cref{proof:lemma_fedcm_rule:client_update_f1,proof:lemma_fedcm_rule:client_update_f2} into \cref{proof:lemma_fedcm_rule:server_update} we have that:
    \begin{align}
        \theta^t&=\theta^{t-1}\left(1-\eta_t\right) + \eta_t \left(
        \theta^{t-1}\left((1-\mu\tilde{\eta}_l)^J + \frac{\hat{\beta}}{\eta_t}\sum_{j=0}^{J-1}\left( 1-\mu\tilde{\eta}_l\right)^j\right)+ (-1)^t \tilde{\eta}_l G \sum_{j=0}^{J-1}\left(1-\mu\tilde{\eta}_l\right)^j - \frac{\hat{\beta}}{\eta_t}\theta^{t-2}\sum_{j=0}^{J-1}\left( 1-\mu\tilde{\eta}_l\right)^j
        \right) \\
        &=\theta^{t-1}\left(1+\hat{\beta}\sum_{j=0}^{J-1}\left(1-\mu\tilde{\eta}_l\right)^j-\eta_t +\eta_t(1-\mu\tilde{\eta}_l)^J\right) + (-1)^t \eta_t\tilde{\eta}_l G \sum_{j=0}^{J-1}(1-\mu\tilde{\eta}_l)^j - \hat{\beta} \theta^{t-2}\sum_{j=0}^{J-1}\left(1-\mu\tilde{\eta}_l\right)^j
    \end{align}
\end{proof}
\begin{corollary}[One round progress for \fedavgm{} and \fedcm{} with a single local step]
\label{corollary:fedavgm_fedcm}
    For any positive constants $G, \mu$, define $\mu$-strongly convex functions $f_1(\theta) := \frac{\mu}{2} \theta^2 + G\theta$ and $f_2(\theta):=\frac{\mu}{2} \theta^2 - G\theta$ satisfying assumption \ref{assum:bounded_gd} and such that $f(\theta)=\frac{1}{2}\left(f_1(\theta) + f_2(\theta)\right)$. Under cyclic participation (assumption \ref{assum:cyclic_part}) with $C=0.5$, for any $t \ge 1$ the update of \fedcm{} and \fedcm{} after a single local step, with step-size $\eta_t$ and momentum weight $\beta$ is:
    \begin{equation*}
        \theta^t=\theta^{t-1}\left( 1-\mu\tilde{\eta}_t + \beta\right) - \beta \theta^{t-2} + (-1)^t\tilde{\eta}_t G
    \end{equation*}
    where $\tilde{\eta}_t=\eta_t(1-\beta)$.
\end{corollary}
\begin{proof}
    Let us first define $D_{\sc Mom}(s,J,r):=\sum_{j=s}^{J}r^j$ and notice that:
    \begin{align}
        D_{\sc Mom}(0,0,r)&=1
        ,\qquad
        D_{\sc Mom}(J,J)=r^J
        \label{corollary:fedavgm_fedcm:upd_geom}
    \end{align}
    Let us call $\gamma_t$ and $\gamma_l$ global and local step-sizes in \fedavgm{} and \fedcm{}, and define an effective step $\eta_t=\gamma_t\gamma_l$.
    From \cref{lemma:update_fedavgm}, for \fedavgm{} we have that:
    \begin{align}
     \theta^t &= \theta^{t-1}\left(1+\beta-\tilde{\gamma}_t+\tilde{\gamma}_tD_{\sc Mom}(J,J,1-\mu\tilde{\gamma}_l)\right) + (-1)^{t}\tilde{\gamma}_t \gamma_l G D_{\sc Mom}(0,J-1,1-\mu\tilde{\gamma}_l) - \beta \theta^{t-2} \\
     &\overset{\text{\ref{corollary:fedavgm_fedcm:upd_geom}}}{=} \theta^{t-1}\left(1+\beta-\tilde{\gamma}_t+\tilde{\gamma}_t(1-\mu{\gamma}_l\right) + (-1)^{t}\tilde{\gamma}_t \gamma_l G - \beta \theta^{t-2} \\
     &= \theta^{t-1}\left(1+\beta-\mu\tilde{\eta}_t\right) + (-1)^{t}\tilde{\eta}_t G - \beta \theta^{t-2}
    \end{align}
    Similarly, from \cref{lemma:update_fedcm}, for \fedcm{} we have that:
    \begin{align}
     \theta^t &= \theta^{t-1}\left(1+\hat{\beta}D_{\sc Mom}(0,J-1,1-\mu\tilde{\gamma}_l)-\gamma_t+\gamma_t D_{\sc Mom}(J,J,1-\mu\tilde{\gamma}_l)\right) \nonumber\\ &+ (-1)^{t}\gamma_t \tilde{\gamma}_l G D_{\sc Mom}(0,J-1,1-\mu\tilde{\gamma}_l) - \hat{\beta} \theta^{t-2}D_{\sc Mom}(0,J-1,1-\mu\tilde{\gamma}_l) \\
     &\overset{\text{\ref{corollary:fedavgm_fedcm:upd_geom}}}{=} \theta^{t-1}\left(1+\hat{\beta}-\gamma_t+\gamma_t (1-\mu\tilde{\gamma}_l)\right) + (-1)^{t}\gamma_t \tilde{\gamma}_l G - \hat{\beta} \theta^{t-2} \\
     &= \theta^{t-1}\left(1+\beta-\mu\tilde{\eta}_t\right) + (-1)^{t} \tilde{\eta}_t G - \beta \theta^{t-2}
    \end{align}
\end{proof}

\subsubsection{Proof of \cref{lemma:lb_construction}} (\fedavgm{} and \fedcm{} on two one-dimensional clients)
\label{proof:lemma:lb_construction}

    We assume each client is assigned one of the two below simple one-dimensional functions for any given $\mu$ and $G$, and assume functions are sampled cyclically, {\ie}:
    \begin{align}
        f^t(\theta) &:= \left\{ 
        \begin{alignedat}{2}
            &f_1(\theta) := \frac{\mu}{2} \theta^2 + G\theta &\quad& \text{if $t$ is odd} \\
            &f_2(\theta):=\frac{\mu}{2} \theta^2 - G\theta  &\quad& \text{otherwise}
        \end{alignedat}
        \right.
    \end{align}
    Both functions are $\mu$-strongly convex and $f(\theta)=\frac{1}{2}\left(f_1(\theta) + f_2(\theta)\right)=\frac{\mu}{2}(\theta)^2$, which has global minimizer at $\theta^*=0$.
    By \cref{lemma:update_fedavgm,lemma:update_fedcm}, the update of \fedavgm{} and \fedcm{} can be written as:
\begin{align}
        \theta^t \leftarrow p_t^{(a)}\theta^{t-1}+q_t^{(a)}(-1)^tG - r_t^{(a)}\theta^{t-2} 
\end{align}
    We can formalize the analysis of the above as a second-order discrete-time linear system using state-space representation.
    A discrete-time linear system can be represented in state-space form as:
    \begin{align}
    \left\{\begin{alignedat}{2}
    \label{proof:fedcmlb:ssf}
        \mathbf{z}[t] &= \mathbf{A}[t]\mathbf{z}[t-1] + \mathbf{B}\mathbf{u}[t] \\
        \mathbf{y}[t] &= \mathbf{C}\mathbf{z}[t] 
    \end{alignedat}\right.
    \end{align}
    where:
    \begin{align*}
        & \mathbf{z}[t] =
            \begin{pmatrix}
                z_1[t] &
                z_2[t]
            \end{pmatrix}^\top=
            \begin{pmatrix}
                \theta^t &
                \theta^{t-1}
            \end{pmatrix}^\top,
        && \mathbf{u}[t] =
            \begin{pmatrix}
            (-1)^t q_t^{(a)} G 
            \end{pmatrix} \\
        & \mathbf{A}[t]=
            \begin{pmatrix}
                p_t^{(a)} & -r_t^{(a)} \\
                1                       & 0
            \end{pmatrix}
        && \mathbf{B} = 
            \begin{pmatrix}
                1 &
                0
            \end{pmatrix}^\top,
        && \mathbf{C} =
            \begin{pmatrix}
                1 &
                0
            \end{pmatrix}
        &&%
    \end{align*}
    Given an initial state condition $\mathbf{z}[1]=\begin{pmatrix}
        \theta^1 & \theta^0
    \end{pmatrix}^\top$, with $\theta^1=\theta^0$, the result of the lemma follows from unrolling the recursion and defining the state transition matrix $\Psi(t,k):=\prod_{s=k+1}^t\mathbf{A}[s]$.

\subsubsection{Proof of \cref{thm:lb_constant_step}}(Lower Bound under Constant Step-size)

\label{proof:thm:lb_constant_step}
    Let $\mathbf{z}[t]$ and $\mathbf{y}[t]$ be the state-space representation and the output of the discrete linear time-invariant (LTI) system constructed in \cref{proof:fedcmlb:ssf} of \cref{lemma:lb_construction}, and consider only one local step (\ie{} $J=1$). By \cref{corollary:fedavgm_fedcm}, this means analyzing the system in \cref{proof:fedcmlb:ssf} with state and input matrices as:
    \begin{align}
        &\mathbf{A}[t]=
            \begin{pmatrix}
                1+\beta-\mu\tilde{\eta}_t & -\beta \\
                1                       & 0
            \end{pmatrix},
        &\mathbf{u}[t]=
            \begin{pmatrix}
                (-1)^t \tilde{\eta}_tG
            \end{pmatrix}
    \end{align}
    We denote $\mathbf{y}_{ZIR}[t]:=\mathbf{C}\Psi(t,1) \mathbf{z}[1]\vphantom{\sum_{k=1}^{t-1}}$ as the \textbf{zero-input response} and $\mathbf{y}_{ZSR}[t]:=\mathbf{C}\sum_{k=1}^{t}\Psi(t,k)\mathbf{B}\mathbf{u}[k]$ as the \textbf{zero-state response}, which can be studied separately thanks to linearity.
    We assume a constant step size, i.e., $\eta_t = \eta$ (or equivalently $\tilde{\eta}_t = \tilde{\eta}$) $\forall t$.
    \paragraph{Solution of zero-input response}
    
    Under constant learning rate the state matrix $\mathbf{A}[t]$ is $\mathbf{A}[t]=\mathbf{A} \,\forall t$. Therefore the state-transition matrix  becomes $\Psi(t,1)=\mathbf{A}^{t-1}$ and we have that $\mathbf{A}^{t-1} \mathbf{z}[1] \to 0 \iff \mathbf{A}^{t-1} \to 0$ as $t \to \infty$
    for any given initial state $\mathbf{z}[1]\neq \mathbf{0}$. The asymptotic convergence of the response depends on the eigenvalues of the matrix $\mathbf{A}$ being strictly less than one. The eigenvalues of $\mathbf{A}$ are the solutions $\lambda_{1,2}$ to the associated characteristic equation, and to find the values of $\eta,\beta$ which satisfy the condition we apply the Jury stability criterion: 
    \begin{align}
        P(\lambda) :=\det(\lambda \mathbf{I} -\mathbf{A}) &=
        \det \begin{pmatrix}
                \lambda - (1+\beta-\mu\tilde{\eta}) & \beta \\
                -1           & \lambda
            \end{pmatrix} \\
            &= \lambda\left(\lambda - (1+\beta-\mu\tilde{\eta})\right) + \beta \\
            &= \lambda^2 - (1+\beta-\mu\tilde{\eta})\lambda + \beta
    \end{align}
    \begin{align}
        \bullet \;\text{\underline{\emph{condition 1:}}} \; &\mathbf{\beta < |1|}: \beta < 1 \Rightarrow \beta \in [0,1) \\
        \bullet \;\text{\underline{\emph{condition 2:}}} \; &\mathbf{P(1) > 0}: \nonumber \\
            &\Rightarrow  1 -(1+\beta-\mu\tilde{\eta}) + \beta > 0 
            &\Rightarrow & \mu(1-\beta)\eta > 0 &
            \Rightarrow &\eta>0 \\
        \bullet \;\text{\underline{\emph{condition 3:}}} \; &\mathbf{P(-1) > 0}: \nonumber\\
            &\Rightarrow  1 + (1+\beta-\mu\tilde{\eta}) + \beta > 0 &
            \Rightarrow & 2(1+\beta) > \mu\tilde{\eta} &
            \Rightarrow & \eta < \frac{2(1+\beta)}{\mu(1-\beta)}
        \end{align}

    In the above steps we have used the definition $\tilde{\eta} := (1-\beta)\eta$ from \cref{lemma:lb_construction}.
    Summarizing, under the condition
    \begin{equation}
        \eta \in \left( 0, \frac{2(1+\beta)}{\mu(1-\beta)}\right)
        \quad \text{with } \beta \in [0,1)
        \label{eq:cond_lr}
    \end{equation}
    the norm of $\mathbf{y}_{ZIR}[t]$ is monotonically decreasing w.r.t. $t$ and converges to zero as $t \rightarrow \infty$.
    \paragraph{Solution of the zero-state response}
    
    Proceeding with the analysis of the zero-state response $\mathbf{y}_{ZSR}[t]$, we show that the presence of the periodic term (due to the cyclic client switching) induces an oscillatory dynamic that does not decrease to zero and that depends on $G$. 
    Since the input is 2-periodic, the zero-state response converges to a limit cycle of the same period. Namely, for a some fixed $\mathbf{c} \in \mathbb{R}^2$, we search for a solution of the \textbf{periodic form} $\mathbf{z}[t]=(-1)^t \mathbf{c}$:
    \begin{align}
     \mathbf{z}[t] &= \mathbf{A}\mathbf{z}[t-1] + \mathbf{B}\mathbf{u}[t] \\
    \overset{\text{periodic form}}{\Rightarrow}
    (-1)^t \mathbf{c} &= \mathbf{A}(-1)^{t-1}\mathbf{c} + (-1)^t \tilde{\eta}G \mathbf{B}\\
    \overset{\text{division by } (-1)^t}{\Rightarrow} \quad\;\; \mathbf{c}&= -\mathbf{A}\mathbf{c} + \tilde{\eta}G\mathbf{B} \\
   \overset{\text{group } \mathbf{c} \text{ to l.h.s}}{\Rightarrow}\qquad \mathbf{c}&= (\mathbf{I}+\mathbf{A})^{-1}\tilde{\eta}G\mathbf{B} \\
      &= \begin{pmatrix}
        \frac{\eta(1-\beta)G}{2(1+\beta) - \mu\eta(1-\beta)} & -\frac{\eta(1-\beta)G}{2(1+\beta) - \mu\eta(1-\beta)}
    \end{pmatrix}^\top
    \end{align}
    This yields that:
    \begin{align}
        \mathbf{y}_{ZSR}[t]&= \mathbf{Cz}[t] = \left(1 \;\; 0\right)  \mathbf{z}[t] \\
        &= (-1)^t  \frac{\eta(1-\beta)G}{2(1+\beta) - \mu\eta(1-\beta)}\\
        \Rightarrow \; \big|\mathbf{y}_{ZSR}[t]\big| &= \frac{\eta(1-\beta)G}{2(1+\beta) - \mu\eta(1-\beta)} \label{eq:lim_norm_theta}
    \end{align}

    \paragraph{Lower and Upper bounds}
    Combining the previous results, we have that $\mathbf{y}[t] = \mathbf{y}_{ZIR}[t] + \mathbf{y}_{ZSR}[t]$.
    The first term in the r.h.s. starts at $\theta^0$ and under condition in \cref{eq:cond_lr} is converging exponentially to zero. The second term is periodic and the amplitude of the limit cycle increases monotonically with the learning rate $\eta$ (see \cref{eq:lim_norm_theta}).
    Choosing a small enough value of $\eta$ which satisfies the condition (\ref{eq:cond_lr}), \eg{}
    \begin{subequations}
        \begin{empheq}[right=\empheqrbrace{\qquad\text{with } T>1,\,0<c_1<c_2\leq2}]{align}
            \eta &>\frac{c_1}{\mu T}\left(\frac{1+\beta}{1-\beta}\right) \label{proof:lb_constant_eta1}\\
            \eta &< \frac{c_2}{\mu T}\left(\frac{1+\beta}{1-\beta}\right) \label{proof:lb_constant_eta2}
        \end{empheq}
    \end{subequations}

    we have that:
    \begin{align}
        |\theta^\infty| &= \lim_{t\to\infty} |\theta^t| = \lim_{t\to\infty} \big|\underbrace{\mathbf{y}_{ZIR}[t]}_{\text{vanishing}}  + \underbrace{\mathbf{y}_{ZSR}[t]}_{\text{periodic}} \big|= \big|\mathbf{y}_{ZSR}[t]\big|\\
        \theta^\infty
        &\overset{\text{inject (\ref{proof:lb_constant_eta1}) in (\ref{eq:lim_norm_theta})}}{\geq} \frac{c_1(1+\beta)}{\mu T(1-\beta)} \frac{(1-\beta)G}{2(1+\beta) - \frac{c_1}{T}(1+\beta)} \\
        &= \frac{c_1 G}{\mu(2T-c_1)} \geq \Omega\left(  \frac{G}{\mu T}\right) \\
        \theta^\infty &\overset{\text{inject (\ref{proof:lb_constant_eta2}) in (\ref{eq:lim_norm_theta})}}{\leq} \frac{c_2(1+\beta)}{\mu T(1-\beta)} \frac{(1-\beta)G}{2(1+\beta) - \frac{c_2}{T}(1+\beta)} \\
        &= \frac{c_2 G}{\mu(2T-c_2)} \leq \mathcal{O}\left(  \frac{G}{\mu T}\right)
    \end{align}

We finish the proof by noting that $f(\theta) =\frac{\mu}{2}\theta^2$, with minimum $f(\theta^*)=0$ at $\theta^*=0$.\\

\subsubsection{Proof of \cref{thm:lb_decreasing_step}}(Lower Bound under Decreasing Step-size)
\label{proof:thm:lower_bound_decreasin_stepsize}

To study the original system from eq. (\ref{proof:fedcmlb:ssf}), we first split matrix $A[t]$ in two terms:
\begin{equation}
\mathbf{A}[t]=
    \begin{pmatrix}
        1+\beta-\mu\tilde{\eta}_t & -\beta \\
        1                       & 0
    \end{pmatrix} =
    \underbrace{\begin{pmatrix}
    1+\beta & -\beta \\
        1                       & 0
    \end{pmatrix}}_{:=\mathbf{A}^\infty} + 
    \underbrace{\begin{pmatrix}
        -\frac{\mu\tilde{\eta}}{t^\alpha} & 0 \\
        0 & 0
    \end{pmatrix}}_{:=\mathbf{E}[t]}
\end{equation}

With this notation, the system takes the following form:
\begin{equation}
    \label{eq:system_Ainf_E}
    \mathbf{z}[t] = \left( \mathbf{A}^\infty + \mathbf{E}[t]\right)\mathbf{z}[t-1] +\mathbf{B}\mathbf{u}[t]
\end{equation}
Since the system is time-variant, we cannot directly use the eigenvalues of $\mathbf{A}^\infty$ to analyze its stability and we will need to look at the evolution of the state. To this end, we first transform the system by diagonalizing the part corresponding to $\mathbf{A}^\infty$. We have that
    \begin{align}
        \mathbf{A}^\infty &=\mathbf{P}\mathbf{\Lambda}\mathbf{P}^{-1} &
        \mathbf{P}&=\begin{pmatrix}
            1 & \beta \\
            1 & 1
        \end{pmatrix} &
        \mathbf{\Lambda} &= \begin{pmatrix}
            1 & 0 \\
            0 & \beta
        \end{pmatrix} &
        \mathbf{P}^{-1}=\frac{1}{\beta-1}\begin{pmatrix}
            -1 & \beta \\
            1  & -1
        \end{pmatrix}
    \end{align}
and we transform the system (\ref{eq:system_Ainf_E}) as follows:

    \begin{align}
        \bar{\mathbf{z}}[t] &= \mathbf{P}^{-1}\mathbf{z}[t]  \\  
        &= \mathbf{P}^{-1}(\mathbf{A}^{\infty} + \mathbf{E}[t])\mathbf{P} \bar{\mathbf{z}}[t-1] + \mathbf{P}^{-1}\mathbf{B}\mathbf{u}[t]\\
        &= \big(\underbrace{\mathbf{P}^{-1}\mathbf{A}^{\infty}\mathbf{P}}_{\mathbf{\Lambda}} + \underbrace{\mathbf{P}^{-1}\mathbf{E}[t]\mathbf{P}}_{:=\mathbf{H}[t]}\big) \bar{\mathbf{z}}[t-1] + \underbrace{\mathbf{P}^{-1}\mathbf{B}}_{:=\mathbf{W}}\mathbf{u}[t] \\
        &=(\mathbf{\Lambda} + \mathbf{H}[t])\bar{\mathbf{z}}[t-1] +\mathbf{W}\mathbf{u}[t]
    \end{align}
with
\begin{align}
\mathbf{H}[t] &= -\frac{\mu\tilde{\eta}}{(\beta-1)t^\alpha}\begin{pmatrix}
            -1 & \beta \\ 1 & -1
        \end{pmatrix}
        \begin{pmatrix}
            1 & 0 \\ 0 & 0
        \end{pmatrix}
        \begin{pmatrix}
            1 & \beta \\ 1 & 1                
        \end{pmatrix} \\
        &= -\frac{\mu\tilde{\eta}}{(\beta-1)t^\alpha}
        \begin{pmatrix}
            -1 & -\beta \\ 1 & \beta
        \end{pmatrix} \\
        \mathbf{W} &=\frac{1}{1-\beta}\begin{pmatrix}
            1 \\ -1
        \end{pmatrix}
\end{align}
    
    This leads to:
    \begin{equation}
        \begin{aligned}
        \bar{\mathbf{z}}[t] &= \left[
        \begin{pmatrix}
            1 & 0 \\
            0 & \beta
        \end{pmatrix}
        - \frac{\mu\tilde{\eta}}{(1-\beta)t^\alpha}
        \begin{pmatrix}
            1 & \beta \\
            -1 & -\beta
        \end{pmatrix}
        \right] \bar{\mathbf{z}}[t-1] + \frac{1}{1-\beta}\begin{pmatrix}
            1 \\ -1
        \end{pmatrix} \mathbf{u}[t]\\
        &= \left[
        \begin{pmatrix}
            1 & 0 \\
            0 & \beta
        \end{pmatrix}
        - \frac{\mu\eta}{t^\alpha}
        \begin{pmatrix}
            1 & \beta \\
            -1 & -\beta
        \end{pmatrix}
        \right] \bar{\mathbf{z}}[t-1] + 
        \frac{\eta}{t^\alpha}
        \begin{pmatrix}
            1 \\ -1
        \end{pmatrix} G(-1)^t
        \end{aligned}
    \end{equation}
where we used the definitions $\tilde{\eta}=(1-\beta)\eta$ and $\mathbf{u}[t] = (-1)^t \tilde{\eta}_t G$ from \cref{lemma:lb_construction}.
We proceed by explicitly writing the transformed-state equation component-wise, i.e.,
\begin{subequations}
    \begin{empheq}[left=\empheqlbrace]{align}
        \bar{z}_1[t] &= \left(1-\frac{\mu\eta}{t^\alpha} \right)\bar{z}_1[t-1] -\underbrace{\frac{\mu\eta\beta}{t^\alpha}\bar{z}_2[t-1]}_{:=r_1[t]} + \underbrace{\frac{\eta}{t^\alpha}G (-1)^t}_{:=r_3[t]}\label{eq:z1_bar}\\
        \bar{z}_2[t] &= \left(\beta + \frac{\mu\eta\beta}{t^\alpha} \right)\bar{z}_2[t-1] + \underbrace{\frac{\mu\eta}{t^\alpha}\bar{z}_1[t-1]}_{:=r_2[t]} - \underbrace{\frac{\eta}{t^\alpha}G (-1)^t}_{:=r_3[t]}
        \label{eq:z2_bar}
    \end{empheq}
\end{subequations}

Now, we unroll these expressions back to the time $t=0$. Specifically, for $\bar{z}_1[t]$ we have:
\begin{align}
\bar{z}_1[t] &= \left(1-\frac{\mu\eta}{t^\alpha} \right)\bar{z}_1[t-1] -r_1[t] + r_3[t]\\
&= \left(1-\frac{\mu\eta}{t^\alpha} \right)
\left[\left(1-\frac{\mu\eta}{(t-1)^\alpha} \right)\bar{z}_1[t-2] -r_1[t-1] + r_3[t-1]\right]
-r_1[t] + r_3[t] \\
&=\left(1-\frac{\mu\eta}{t^\alpha}\right)\left(1-\frac{\mu\eta}{(t-1)^\alpha}\right)\bar{z}_1[t-2]
-\left(1-\frac{\mu\eta}{t^\alpha}\right)r_1[t-1]-r_1[t] \\ &+ \left(1-\frac{\mu\eta}{t^\alpha}\right)r_3[t-1] + r_3[t] \\
&\qquad\vdots\\
&= \prod_{k=2}^{t}\left(1-\frac{\mu\eta}{k^\alpha}\right)\bar{z}_1[1]
+ \sum_{s=2}^{t} \prod_{k=s+1}^{t}\left(1-\frac{\mu\eta}{k^\alpha}\right) \big(r_3[s] - r_1[s]\big)
\end{align} 
Using similar steps for $\bar{z}_2[t]$ (omitted here for brevity), and defining the shorthand expressions
\begin{align}
    \Psi_1(t,s,\alpha) & := \prod_{k=s+1}^{t}\left(1-\frac{\mu{\eta}}{k^\alpha}\right) \\
    \Psi_2(t,s,\alpha) &:= \prod_{k=s+1}^{t}\left(\beta+\frac{\mu{\eta}\beta}{k^\alpha}\right)
\end{align}

we finally rewrite the original system as

\begin{subequations}
    \label{proof:lb_decreasing_step:z_unrolled}
    \begin{empheq}[left=\empheqlbrace]{align}
        \bar{z}_1[t] &= \Psi_1(t,1,\alpha)\bar{z}_1[1] - \sum_{s=2}^{t}\Psi_1(t,s,\alpha)r_1[s] + \sum_{s=2}^{t}\Psi_1(t,s,\alpha) r_3[s] \label{proof:lb_decreasing_step:z1_unrolled} \\
        \bar{z}_2[t] &= \Psi_2(t,1,\alpha)\bar{z}_2[1] + \sum_{s=2}^{t}\Psi_2(t,s,\alpha)r_2[s] - \sum_{s=2}^{t}\Psi_2(t,s,\alpha) r_3[s] \label{proof:lb_decreasing_step:z2_unrolled}
    \end{empheq}
\end{subequations}
Since $\bar{z}_1[t], \bar{z}_2[t]$ are coupled in the system in \cref{proof:lb_decreasing_step:z_unrolled}, in the following we use a technique based on a self-consistent ansatz. That is, we assume an asymptotic form for $\bar{z}_1[t]$ and then verify that the resulting solution for $\bar{z}_2[t]$ leads to a conclusion consistent with the hypothesis.
Since the behavior of the system substantially changes when $\alpha=1$ and $\alpha>1$, we separately analyze the three cases.
\paragraph{Convergence for $0 < \alpha < 1$}
Starting from $\bar{z}_2[t]$, we analyze it assuming $\bar{z}_1[t] \sim c_1(-1)^t/t^\epsilon$, for some arbitrarily small $\epsilon > 0$ and some constant $c_1>0$. Under this assumption, from \cref{proof:lb_decreasing_step:z2_unrolled} we have that:
\begin{align}
    &\lim_{t \to \infty} \bar{z}_2[t] = \nonumber \\
    &\lim_{t \to \infty} 
    \left[
        \Psi_2(t,1,\alpha)\bar{z}_2[1] + \mu\eta\sum_{s=2}^{t}\Psi_2(t,s,\alpha)\frac{1}{s^\alpha}\bar{z}_1[s-1] - \eta G\sum_{s=2}^{t}\Psi_2(t,s,\alpha) \frac{(-1)^s}{s^\alpha}
    \right] \label{proof:lb_decreasing_step:z2_conv_base} \\
    \overset{\text{\ref{corollary:lim_trans_matrix_2}}}{=}& \lim_{t \to \infty} 
    \left[
        \beta^t \bar{z}_2[1] - \mu\eta c_1\sum_{s=2}^{t}\Psi_2(t,s,\alpha)\frac{(-1)^s}{s^{\alpha+\epsilon}} - \eta G\sum_{s=2}^{t}\Psi_2(t,s,\alpha) \frac{(-1)^s}{s^\alpha}
    \right]  \\
    \overset{\text{\ref{lemma:lim_sum_alt_trans_matrix_2}}}{=}& \lim_{t \to \infty} 
    \left[
        \beta^t \bar{z}_2[1] - \mu\eta c_1\frac{(-1)^t}{t^{\alpha+\epsilon}} 
        - \eta G\frac{(-1)^t}{t^\alpha}
    \right] \\
    =& \lim_{t \to \infty} \eta G\frac{(-1)^{t+1}}{t^\alpha} \label{proof:lb_decreasing_step:z2_finalrate_first}
\end{align}
Where in the second passage we substituted the hypothesis for $\bar{z}_1[t]$ and in third passage we used \cref{lemma:lim_sum_alt_trans_matrix_2} twice, with $n=\alpha+\epsilon$ for the second term and $n=\alpha$ for the third term. 
In the last passage we considered that, since $\epsilon>0$, the third term is asymptotically slower than both the first and the second.
Using the results obtained for $\bar{z}_2[t]$, proceeding from \cref{proof:lb_decreasing_step:z1_unrolled} we have that:
\begin{align}
    &\lim_{t \to \infty} \bar{z}_1[t] = \nonumber \\
    &\lim_{t \to \infty} 
    \left[
        \Psi_1(t,1,\alpha)\bar{z}_1[1] - \mu\eta\beta\sum_{s=2}^{t}\Psi_1(t,s,\alpha)\frac{1}{s^\alpha}\bar{z}_2[s-1] + \eta G\sum_{s=2}^{t}\Psi_1(t,s,\alpha) \frac{(-1)^s}{s^\alpha}
    \right] \label{proof:lb_decreasing_step:z1_conv_base}\\
    \overset{\text{\ref{corollary:lim_trans_matrix_1}}}{=}&\lim_{t \to \infty} 
    \left[
        \exp\left( -t^{1-\alpha}\right)\bar{z}_1[1] - \mu\eta^2\beta G\sum_{s=2}^{t}\Psi_1(t,s,\alpha)\frac{(-1)^s}{s^{2\alpha}} + \eta G\sum_{s=2}^{t}\Psi_1(t,s,\alpha) \frac{(-1)^s}{s^\alpha}
    \right] \\
    \overset{\text{\ref{lemma:lim_sum_alt_trans_matrix_gen}}}{=}&\lim_{t \to \infty} 
    \left[
        \exp\left( -t^{1-\alpha}\right)\bar{z}_1[1] - \mu\eta^2\beta G \frac{(-1)^t}{t^{2\alpha}} + \eta G \frac{(-1)^t}{t^\alpha}
    \right] \\
    =&\lim_{t \to \infty} \eta G \frac{(-1)^t}{t^\alpha} \label{proof:lb_decreasing_step:z1_finalrate_first}
\end{align}
Where in the second passage we substituted the result for $\bar{z}_2[t]$ and in third passage we used \cref{lemma:lim_sum_alt_trans_matrix_gen} twice, with $n=2\alpha$ for the second term and $n=\alpha$ for the third term. 
So, for $0 <\alpha < 1$, assuming $\bar{z}_1[t] \sim c_1 (-1)^t/t^\epsilon$ leads to the conclusion that $\bar{z}_1[t] \to \eta G (-1)^t/t^{\alpha}$, so the assumption is valid for $\epsilon=\alpha$ and $c_1=\eta G$, and any substitution with $\epsilon \in (0,\alpha)$ is valid.

\paragraph{Convergence for $\alpha = 1$}
Similarly as before, starting from the assumption $\bar{z}_1[t] \sim (c_1 - c_2(-1)^t)/t^\epsilon$, from \cref{proof:lb_decreasing_step:z2_conv_base} we have that:
\begin{align}
    &\lim_{t \to \infty} \bar{z}_2[t] = \nonumber \\
    \overset{\text{\ref{corollary:lim_trans_matrix_2}}}{=}& \lim_{t \to \infty} 
    \left[
        \beta^t \bar{z}_2[1] + \mu\eta \sum_{s=2}^{t}\Psi_2(t,s,1)\frac{c_1+c_2(-1)^s}{s^{1+\epsilon}} - \eta G\sum_{s=2}^{t}\Psi_2(t,s,1) \frac{(-1)^s}{s}
    \right]  \\
    \overset{\text{\ref{lemma:lim_sum_alt_trans_matrix_2}}}{=}& \lim_{t \to \infty} 
    \left[
        \beta^t \bar{z}_2[1] + \mu\eta c_1\frac{1}{t^{1+\epsilon}} + \mu\eta c_2 \frac{(-1)^t}{t^{1+\epsilon}} 
        - \eta G\frac{(-1)^t}{t}
    \right] \\
    =& \lim_{t \to \infty} \eta G\frac{(-1)^{t+1}}{t} \label{proof:lb_decreasing_step:z2_finalrate_second}
\end{align}
Using the results obtained for $\bar{z}_2[t]$, proceeding from \cref{proof:lb_decreasing_step:z1_conv_base} we have that:
\begin{align}
    &\lim_{t \to \infty} \bar{z}_1[t] = \nonumber \\
    \overset{\text{\ref{corollary:lim_trans_matrix_1}}}{=}&\lim_{t \to \infty} 
    \left[
        \frac{1}{t^{\mu\eta}}\bar{z}_1[1] - \mu\eta^2\beta G\sum_{s=2}^{t}\Psi_1(t,s,1)\frac{(-1)^s}{s^{2}} + \eta G\sum_{s=2}^{t}\Psi_1(t,s,1) \frac{(-1)^s}{s}
    \right] \\
    \overset{\text{\ref{lemma:lim_sum_alt_trans_matrix_gen}}}{=}&\lim_{t \to \infty} 
    \left[
        \frac{1}{t^{\mu\eta}}\bar{z}_1[1] - \mu\eta^2\beta G \frac{(-1)^t}{t^{\mu\eta}} + \eta G \frac{(-1)^t}{t}
    \right] \\
    =&\lim_{t \to \infty} 
    \left\{
    \begin{alignedat}{2}
        &    
        \frac{1}{t^{\mu\eta}}\bar{z}_1[1] - \mu\eta^2\beta G \frac{(-1)^t}{t^{\mu\eta}} + \eta G \frac{(-1)^t}{t}
        &\quad& \text{if } \eta \in (0,1/\mu) \\
        &
        \frac{1}{t^{\mu\eta}}\bar{z}_1[1] - \mu\eta^2\beta G \frac{(-1)^t}{t^{\mu\eta}} + \eta G \frac{(-1)^t}{t^{\mu\eta}}
        &\quad& \text{if } \eta \in [1/\mu, 2/\mu) \\
    \end{alignedat} 
    \right.\\
    =&\lim_{t \to \infty} \frac{c_1 - c_2(-1)^t}{t^{\mu\eta}} \label{proof:lb_decreasing_step:z1_finalrate_second}
    \intertext{In particular:}
    &c_1=\bar{z}_1[1],\quad
    c_2=\left\{
    \begin{alignedat}{2}
        &    
        \mu\eta^2\beta G 
        &\quad& \text{if } \eta \in (0,1/\mu) \\
        &
        \mu\eta^2\beta G - \eta G
        &\quad& \text{if } \eta \in [1/\mu, 2/\mu) \\
    \end{alignedat} 
    \right.
\end{align}
Where in the second passage we used \cref{lemma:lim_sum_alt_trans_matrix_gen} twice, with $n=2$ for the second term and $n=1$ for the third term, and considered the constraint $\mu\eta<2$.
In conclusion, for $\alpha = 1$, assuming $\bar{z}_1[t] \sim (c_1 - c_2 (-1)^t)/t^\epsilon$ leads to the conclusion that $\bar{z}_1[t] \to (c_1 - c_2 (-1)^t)/t^{\mu\eta}$, so the assumption is valid for $\epsilon=\mu\eta$ and $c_1,c_2$ as above, and any substitution with $\epsilon \in (0,\mu\eta)$ is valid.

Let us notice that, while it is possible to make $c_2=0$ (\ie{} independent on $G$) by choosing $\beta=(\mu\eta)^{-1}$ and $\eta>1/\mu$ (otherwise resulting in the incompatible requirement $\beta=1$), this does not result in overcoming the dependence on $G$ for the original state $z_1[t]$. In fact, since $z_1[t]=\bar{z}_1[t]+\beta \bar{z}_2[t]$ (\cref{proof:lb_decreasing_step:original_system}), for $\eta>1/\mu$ $\bar{z}_2[t]$ in \cref{proof:lb_decreasing_step:z2_finalrate_second} dominates the rate.
\paragraph{Convergence for $\alpha > 1$}
Similarly as before, starting from the assumption $\bar{z}_1[t] \sim c_1$, from \cref{proof:lb_decreasing_step:z2_conv_base} we have that:
\begin{align}
    &\lim_{t \to \infty} \bar{z}_2[t] = \nonumber \\
    \overset{\text{\ref{corollary:lim_trans_matrix_2}}}{=}& \lim_{t \to \infty} 
    \left[
        \beta^t \bar{z}_2[1] + \mu\eta c_1\sum_{s=2}^{t}\Psi_2(t,s,\alpha)\frac{1}{s^\alpha} - \eta G\sum_{s=2}^{t}\Psi_2(t,s,\alpha) \frac{(-1)^s}{s^\alpha}
    \right]  \\
    \overset{\text{\ref{lemma:lim_sum_alt_trans_matrix_2}}}{=}& \lim_{t \to \infty} 
    \left[
        \beta^t \bar{z}_2[1] + \mu\eta c_1\frac{1}{t^\alpha} 
        - \eta G\frac{(-1)^t}{t^\alpha}
    \right] \\
    =& \lim_{t \to \infty} \frac{\mu\eta c_1 + \eta G(-1)^{t+1}}{t^\alpha} \label{proof:lb_decreasing_step:z2_finalrate_third}
\end{align}
Using the results obtained for $\bar{z}_2[t]$, proceeding from \cref{proof:lb_decreasing_step:z1_conv_base} we have that:
\begin{align}
    &\lim_{t \to \infty} \bar{z}_1[t] = \nonumber \\
    \overset{\text{\ref{corollary:lim_trans_matrix_1}}}{=}&\lim_{t \to \infty} 
    \left[
        c\bar{z}_1[1] - \mu\eta\beta \sum_{s=2}^{t}\Psi_1(t,s,\alpha)\frac{\mu\eta c_1 + \eta G(-1)^s}{s^{2\alpha}} + \eta G\sum_{s=2}^{t}\Psi_1(t,s,\alpha) \frac{(-1)^s}{s^\alpha}
    \right] \\
    \overset{\text{\ref{lemma:lim_sum_trans_matrix_1}+\ref{lemma:lim_sum_alt_trans_matrix_gen}}}{=}&\lim_{t \to \infty} 
    \left[
        c\bar{z}_1[1] - \mu\eta\beta c_1 g(2\alpha) - \mu\eta\beta f(2\alpha)\eta G + \eta G f(\alpha)
    \right]
\end{align}
where $c$ is a positive constant as in \cref{corollary:lim_trans_matrix_2} and $g(\alpha), f(\alpha)$ are functions in $\alpha$, constant in $t$, determining the proper value at convergence of $\bar{z}_1[t]$ (as bounded in \cref{lemma:lim_sum_trans_matrix_1,lemma:lim_sum_alt_trans_matrix_gen}).
Assuming initialization at optimum (\ie{} $\bar{z}_1[1]=0$), we solve for $c_1$:
\begin{align}
    \lim_{t \to \infty} \bar{z}_1[t]&=c_1=\eta G \frac{f(\alpha)-\mu\eta\beta f(2\alpha)}{1+\mu\eta\beta g(2\alpha)} \\
    &\overset{\eta \sim \mathcal{O}(1/\mu)}{=} \Theta\left(\frac{G}{\mu}\right) \label{proof:lb_decreasing_step:z1_finalrate_third}
\end{align}
\paragraph{Convergence of the original system $\mathbf{z}[t]$}
Recalling that $\mathbf{z}[t] = \mathbf{P}\bar{\mathbf{z}}[t]$, we have that:
\begin{align}
    \mathbf{z}[t] = \begin{pmatrix}
            1 & \beta \\
            1 & 1
        \end{pmatrix} \bar{\mathbf{z}}[t] 
        &=
        \begin{pmatrix}
            \bar{z}_1[t] + \beta\bar{z}_2[t] \\
            \bar{z}_1[t] + \bar{z}_2[t]
        \end{pmatrix}
        \label{proof:lb_decreasing_step:original_system}
    \end{align}
So, from \cref{proof:lb_decreasing_step:z2_finalrate_first,proof:lb_decreasing_step:z1_finalrate_first},\cref{proof:lb_decreasing_step:z2_finalrate_second,proof:lb_decreasing_step:z1_finalrate_second} and from \cref{proof:lb_decreasing_step:z2_finalrate_third,proof:lb_decreasing_step:z1_finalrate_third}, we have that:
\begin{align}
    \lim_{t \to \infty} \left|z_1[t] \right|&=
    \left\{
    \begin{alignedat}{3}
        & \Theta\left( \frac{G}{\mu t^\alpha}\right) &\qquad& \text{if } 0<\alpha<1 \\
        & \Theta\left( \frac{G}{\mu t^{\min(\mu\eta,1)}}\right) &\qquad& \text{if } \alpha=1 \\
        & \Theta\left( \frac{G}{\mu}\right) &\qquad& \text{if } \alpha>1 \\
    \end{alignedat}
    \right.
\end{align}
We finish the proof by noting that $f(\theta) =\frac{\mu}{2}\theta^2$, with minimum $f(\theta^*)=0$ at $\theta^*=0$.

\end{document}